\documentclass[preprint,12pt,authoryear]{elsarticle}

\usepackage{soul}
\usepackage{amssymb}
\usepackage{amsmath}
\usepackage{amsthm}
\usepackage{framed}
\usepackage{paralist}
\usepackage{tikz}
\usepackage{xcolor}
\usepackage{nccmath}
\usetikzlibrary{shapes.misc, positioning}
\usepackage{ifthen}
\usepackage{rotating}
\usepackage{tablefootnote}

\newtheorem{definition}{Definition}
\newtheorem{theorem}{Theorem}
\newtheorem{example}{Example}
\newtheorem*{theorem*}{Theorem}

\definecolor{orange}{rgb}{1,0.5,0}
\usepackage[textsize=scriptsize,textwidth=3.5cm]{todonotes}

\usepackage{listings}
\newcommand{\AnsProlog}[0]{\ensuremath{Ans\-Prolog} }
\newcommand{\AnsPrologs}[0]{\ensuremath{Ans\-Prolog}}
\newcommand{\shortasp}[1]{\mathtt{#1}}
\newcommand{\asp}[1]{\noindent\ensuremath{{\small{\shortasp{#1}}}}}

\newcommand{\notf}[1]{\ensuremath{\mathrm{not\ } #1}}
\journal{Engineering Applications of Artificial Intelligence}

\begin{document}

\begin{frontmatter}

\title{Practical Reasoning with Norms \\ for Autonomous Software Agents (Full Edition)}


\author[bath]{Zohreh Shams\corref{cor}\fnref{present}}
\ead{z.shams@bath.ac.uk}
\author[bath]{Marina De Vos}
\ead{m.d.vos@bath.ac.uk}
\author[bath]{Julian Padget}
\ead{j.a.padget@bath.ac.uk}
\author[abdn]{Wamberto~W.~Vasconcelos}
\ead{wvasconcelos@acm.org}

\fntext[present]{Present Address: Computer Laboratory, University of Cambridge, UK  (CB3 0FD)}
\cortext[cor]{Corresponding author.}

\address[bath]{Department of Computer Science, University of Bath, UK (BA2 6AH)}
\address[abdn]{Department of Computing Science, University of Aberdeen, UK (AB24 3UE)}

\begin{abstract}
Autonomous software agents operating in dynamic environments need to constantly reason about actions in pursuit of their goals, while taking into consideration norms which might be imposed on those actions. Normative practical reasoning supports agents making decisions about what is best for them to (not) do in a given situation. 
What makes practical reasoning challenging is  the interplay between goals that agents are pursuing and the norms that the agents are trying to uphold. We offer a formalisation to allow agents to plan for multiple goals and norms in the presence of \emph{durative} actions that can be executed \emph{concurrently}. We compare plans based on decision-theoretic notions (i.e. utility) such that the utility gain of goals and utility loss of norm violations are the basis for this comparison. The set of \emph{optimal\/} plans consists of plans that maximise the overall utility, each of which can be chosen by the agent to execute. We provide an implementation of our proposal in Answer Set Programming, thus allowing us to state the original problem in terms of a logic program that can be queried for solutions with specific properties.
The implementation is proven to be sound and complete. 
\end{abstract}


\end{frontmatter}

\section{Introduction}\label{sec:intro}
Reasoning about what to do~-- known as practical reasoning --~for an agent pursuing different goals is a complicated task. When conducting practical reasoning, the agents might exhibit undesirable behaviour that was not predicted. 
The necessity of controlling undesirable behaviour has given rise to the concept of norms that offer a way to define ideal behaviour for autonomous software agents in open environments. Such norms often define obligations and prohibitions that express what the agent is obliged to do and what the agent is prohibited from doing.

Depending on their computational interpretation, norms can be regarded as \emph{soft\/} or 
\emph{hard constraints}. When modelled as hard constraints, the agent subject to the norms is said to be \emph{regimented}, in which case the agent has no choice but blindly to follow the norms \citep{Esteva2001}. Although regimentation guarantees norm compliance, it greatly restricts agent autonomy. Moreover, having individual goals to pursue, self-interested agents might not want to or might not be able to comply with the norms imposed on them. Conversely, \emph{enforcement\/} approaches, in which norms are modelled as soft constraints, leave the choice  of complying with or violating the norms to the agent. However, in order to encourage norm compliance, there are consequences associated, namely a punishment when agents violate a norm \citep{Lopez2005,Pitt2013} or a reward when agents comply with a norm \citep{Aldewereld2006}. In some approaches  (\emph{e.g.}, \citep{Aldewereld2006,Alrawagfeh2014,Wamberto2011}) there is a utility gain/loss associated with respecting norm or not, whereas in the \emph{pressured 
norm compliance\/} approaches (\emph{e.g.}, \citep{Lopez2005}), the choice to violate a norm or not is determined by how the norm affects the satisfaction or hindrance of the agent's goals. 


Existing work (e.g. \citep{Wamberto2011, Sofia2012, Criado2010, Meneguzzi2015}) on normative practical reasoning using enforcement either consider plan generation or
plan selection where there is a set of pre-generated plans available to the agent.
In these works, the attitude agents have toward norms is often one of compliance, meaning that their plans are often selected or, in some approaches, customised, to ensure norm compliance (e.g., \cite{Kollingbaum2005,Alechina2012,Wamberto2011}). We argue that in certain situations, an agent might be better off violating a norm which, if followed, would make it impossible for the agent to achieve an important goal or complying with a more important norm.

In this paper we set out an approach for practical reasoning that considers norms in both plan generation and plan selection. We extend current work on normative plan generation such that the agent attempts to satisfy a set of potentially conflicting goals in the presence of norms, as opposed to conventional planning problems that generate plans for a single goal \citep{Wamberto2011,Sofia2012}. Such an extension is made on top of STRIPS \citep{Nilsson1971}, the most established planning domain language that lays the foundation of many automated planning languages. Additionally, since in reality the actions are often non-atomic, our model allows for planning with durative actions that can be executed concurrently. Through our practical reasoning process agents consider {\em all\/} plans (i.e., sequences of actions), including those leading to norm compliance and violation; each plan gets an associated overall utility for its sequence of actions, goals satisfied, and norms followed/violated, and agents can decide which of them to pursue by comparing the relative importance of goals and norms via their utilities. The plan an agent chooses to follow is not necessarily norm-compliant; however, our mechanism guarantees that the decision will maximise the overall plan utility, and this justifies the occasional violation of norms as the plan is followed. Both plan generation and plan selection mechanisms proposed in this paper are implemented using Answer Set Programming (ASP) \citep{Gelfond1988}.

ASP is a declarative programming paradigm using logic programs under Answer Set semantics. In this paradigm the user provides a description of a problem and ASP works out how to solve the problem by returning answer sets corresponding to problem solutions. The existence of efficient solvers to generate the answers to the problems provided has increased the use of ASP in different domains of autonomous agents and multi-agent systems such as planning \citep{Lifschitz2002} and normative reasoning \citep{padget2006, Panagiotidi2012}. Several action and planning languages such as event calculus \citep{Kowalski1986}, $\mathcal{A}$ (and its descendants $\mathcal{B}$ and $\mathcal{C}$ \citep{Gelfond1998}, Temporal Action Logics (TAL) \citep{Doherty1998}, have been implemented in ASP \citep{Lee2012,Lee2014}, indicating that ASP is appropriate for reasoning about actions. This provides motive and justification for an implementation of STRIPS \citep{Nilsson1971} that serves as the foundation of our model in ASP.

This paper is organised as follows. First we present a scenario in Section \ref{sec:scenario} which we use to illustrate the applicability of our approach. This is followed by the formal model and its semantics in Section \ref{model}. The computational implementation of the model is provided in Section \ref{sec:implementation}. After the discussion of related work in Section \ref{sec:related}, we conclude in Section \ref{sec:conclusion}.

\section{Illustrative Scenario}\label{sec:scenario}
To illustrate our approach and motivate the normative practical reasoning model in the next section, we consider a scenario in which a software agent acts as a supervisory system in a disaster recovery mission and supports human decision-making in response to an emergency. The software agent's responsibility is to provide humans with different courses of actions available and to help humans decide on which course of actions to follow. In our scenario, the agent is to plan for a group of human actors who are in charge of responding to an emergency caused by an earthquake. The agent monitors the current situation (e.g., contamination of water, detection of shocks, etc.) and devises potential plans to satisfy goals set by human actors. In our scenario we assume the following two goals:
\begin{compactenum} 
\item Running a hospital to help wounded people: this goal is fulfilled when medics are present to offer help and they have access to water and medicines. 
\item Organising a survivors' camp: this is fulfilled when the camp's area is secured and a shelter is built.
\end{compactenum}
We also assume the two following norms that the agent has to consider while devising plans to satisfy the goals above:
\begin{compactenum}
\item It is forbidden to built a shelter within 3 time units of detecting shocks.  The cost of violating this norm is 5 units.
\item It is obligatory to stop water distribution for 2 time units once poison is detected in the water. The cost of violating this norm is 10 units.
\end{compactenum}
The formulation of this scenario is provided in \ref{ForScen}.

\section{A Model for Normative  Practical Reasoning}\label{model}
In this research, we take STRIPS \citep{Nilsson1971} as the basis of our normative practical reasoning model. In STRIPS, a planning problem is defined in terms of an initial state, a goal state 
and a set of operators (e.g. actions). Each operator has a set of preconditions representing the circumstances/context in which the operator can be executed, and a set of 
postconditions that result from applying the operator. Any sequence of actions
that satisfies the goal is a solution to the planning problem. 
In order to capture the features of the normative practical reasoning problem that we 
are going to model, we extend the classical planning problem by: 
\begin{compactenum} [(i)] 
\item replacing atomic actions with \emph{durative} actions: 
often the \emph{nature\/} of the actions is non-atomic, which means that although executed atomically in a state, the system state in which they finish executing is not necessarily the same in which they started \citep{Nunes1997}. Refinement of 
atomic actions to durative actions reflects the real time that a machine takes 
to execute certain actions, which is also known as ``real-time duration" of 
actions \citep{Borger2003}. 
\item Allowing a set of potentially inconsistent goals instead of the conventional single goal: the issue of planning for multiple goals distributed among distinct agents is addressed in collaborative planning. We address this issue for a single agent when handling multiple conflicting goals. 
\item Factoring in norms: having made a case for the importance of norms in Section \ref{sec:intro}, we combine normative and practical reasoning. Just like goals, a set of norms is not necessarily consistent, making it potentially impossible for an agent to comply with all norms imposed on it. 
\end{compactenum}
A solution for a normative practical reasoning problem that features (i), (ii) and (iii) is any sequence of actions that satisfies at least one goal. The agent has the choice of violating or complying with norms triggered by execution of a sequence of actions, while satisfying its goals. However, there may be consequences either way that the agent has to foresee. 

We explain the syntax and semantics of the model in Sections~\ref{syntax}--\ref{semantics}. First, however, we present the architecture of our envisaged system in the next section.

\subsection{Architecture}\label{sec:architecture}
The architecture, depicted in Figure \ref{Arch}, shows how re-planning can be considered if a plan in progress is interrupted due to a change in circumstances. This change can be the result of a change in the environment or unexpected actions of other agents in the system. As is customary in multi-agent systems, the agent will regularly check the viability of its plan. The frequency depends on the type of system the agent is operating in, the agent's commitment to its intentions, and the impact of re-computation on the agent's overall performance. 

When an agent decides that re-planning is in order, it will take the state in which the plan is interrupted as the initial state for the new plan and its current goal set as the goals to plan towards. The current goal set does not have to be the same as the goal set the original plan was devised for. Goals can already be achieved in the interrupted plan, previous goals may no longer be relevant and others may have been added. Even if the goals remain the same, the resulting new optimal plan might change due to changes in the state of the system. Similarly, there might be new norms imposed on the agent that will have to be considered in replanning.

We cater for agents which need to create their own individual plans. However, in doing so, in multi-agent scenarios agents will inevitably interact and interfere with each other's plans. The overview of Figure \ref{Arch} will cater for this in two ways: \begin{inparaenum}[(i)] \item agents will notice the states being changed by other agents (a way of indirect communication) and \item the ``Observations'' will also contain interactions among agents (direct communication)\end{inparaenum}.  Although we do not explore these ideas in the present paper, we envisage that the outcome of the indirect and direct interactions among agents could be represented as norms, establishing, for instance, that a particular agent, in the current context (\emph{i.e.}, power or organisational relationships, description of capabilities and global goals, etc.), is forbidden or obliged to do something.

\begin{figure}[t]
\centering
\includegraphics[scale=0.5]{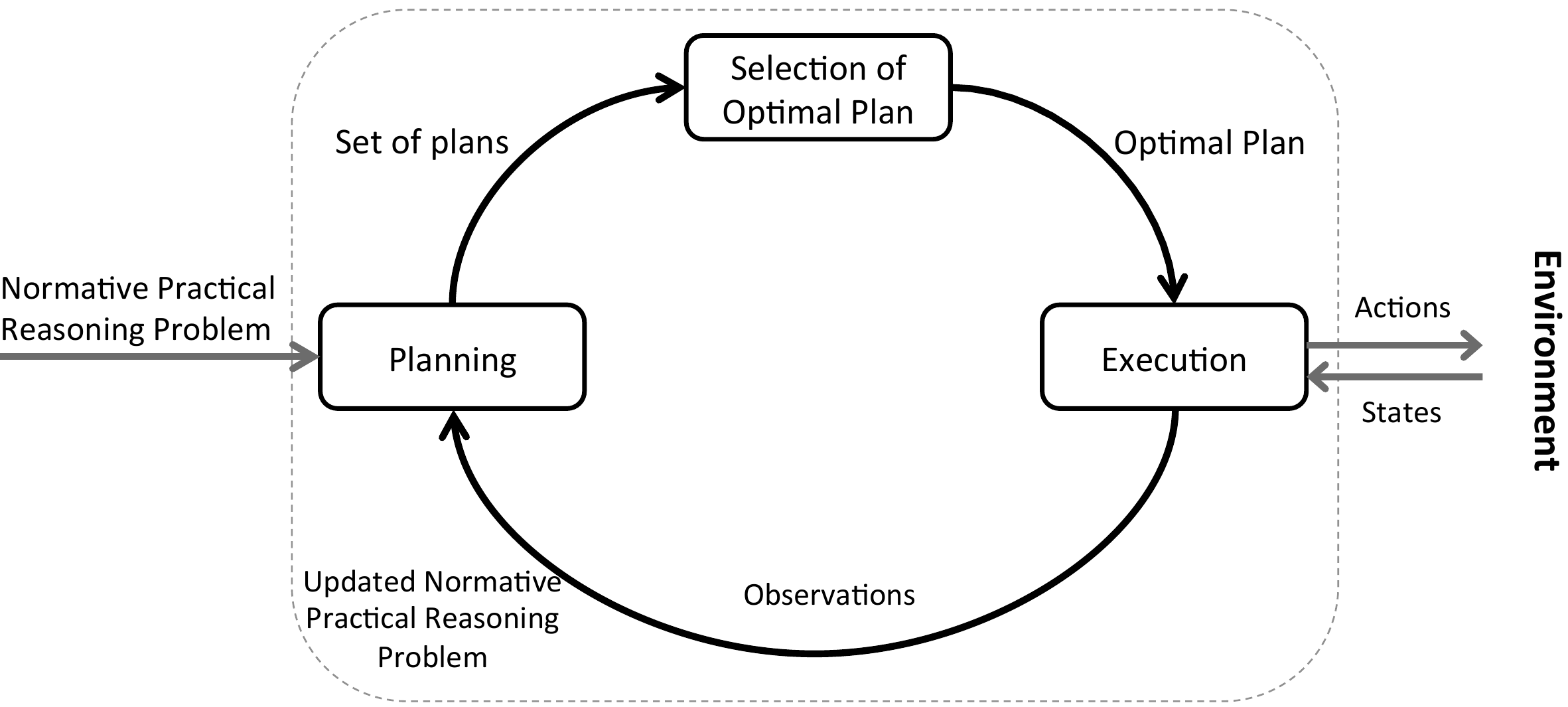}
\caption{Overview of the System Architecture}
\label{Arch}
\end{figure}

\subsection{Syntax}\label{syntax}
We start this section by describing an extended STRIPS planning problem, defined in  \citep{Shams2016},  that accommodates \begin{inparaenum}[(i)]
\item durative actions; \item multiple goals and \item multiple norms.
\end{inparaenum}

\begin{definition}[Normative Practical Reasoning Problem]\label{NPS}
A normative practical reasoning problem is a tuple $P = (\mathit{FL}, \Delta, A, G, N )$ where 
\begin{compactenum}[(i)]
  \item $F\!L$ is a set of fluents;
  \item $\Delta$ is the initial state; 
  \item $A$ is a finite, non-empty set of durative actions for the agent;
  \item $G$ is the agent's set of goals; 
  \item $N$ is a set of norms.
\end{compactenum}
\end{definition}
We describe in the ensuing sections each of the components of a normative practical reasoning problem. 

\subsubsection{Fluents and Initial State}

$F\!L$  is a set of domain fluents describing the domain the agent operates in. A literal $l$ is a fluent or its negation i.e. $l = \mathit{fl}$ or $l = \neg \mathit{fl}$ for some $\mathit{fl} \in \mathit{FL}$. For a set of 
literals $L$, we define $L^{+} = \{\mathit{fl}\,|\,\mathit{fl} \in L\}$ and $L^{-} = \{\mathit{fl}\,|\, \neg\mathit{fl} \in L\}$ to denote the set of positive and negative fluents in 
$L$ respectively. $L$ is well-defined if there exists no fluent $\mathit{fl}\in \mathit{FL}$ such that $\mathit{fl}, \neg \mathit{fl}\in L$, i.e., if $L^{+} \cap L^{-} = \emptyset$.    

The semantics of the normative practical reasoning problem is defined over a set of states $\Sigma$. A state $s \subseteq \mathit{FL}$ is determined by set of fluents that hold \emph{true\/} at a given time, while the other fluents (those that are not present) are considered to be false. A state $s \in \Sigma$ satisfies fluent $\mathit{fl} \in \mathit{FL}$, denoted $s \models \mathit{fl}$, if $\mathit{fl} \in s$. It satisfies its negation $s \models \neg \mathit{fl}$ if $\mathit{fl} \not\in s$. This notation can be extended to a set 
of literals as follows: set $X$ is satisfied in state $s$, $s \models X$,  when
$\forall x \in X, s \models x$.

The set of fluents that hold at the initial state is denoted by $\Delta \subseteq F\!L$.

\subsection{Durative Actions}\label{actions}

The component $A$ of our normative practical reasoning problem $P = (\mathit{FL}, \Delta, A, G, N )$ is a set of durative actions. A durative action has pre- and post-conditions. The effects of an action (as captured by its post-conditions) are not immediate, that is, it takes a non-zero period of time for the effects of an action to take place. 

\begin{definition}[Durative Actions]\label{acts}
A durative action $a$ is a tuple $\langle \mathit{pr}, \mathit{ps}, d \rangle$ where $\mathit{pr}$ and $\mathit{ps}$ are possibly empty and finite sets of well-defined literals representing respectively the pre- and post-conditions of the action, and $d\in\mathbb{N}, d> 0$, is the duration of the action. 
\end{definition} 
Given an action $a = \langle \mathit{pr}, \mathit{ps}, d \rangle$ we may also refer to its three components $\mathit{pr}(a), \mathit{ps}(a)$ and $d_a$. Moreover, we use $\mathit{pr}(a)^{+}$ and $\mathit{pr}(a)^{-}$ to refer to, respectively, the positive and negative literals in $\mathit{pr}(a)$; similarly, we have $\mathit{ps}(a)^{+}$ and $\mathit{ps}(a)^{-}$ to refer to respectively the positive and negative literals in $\mathit{ps}(a)$.

An action $a$ can be executed in a state $s$ if its preconditions hold in that state (i.e. $s \models pr(a)$). When modelling durative actions, there might be several states between the start and end state of the action, during which the action is said to be \emph{in progress}. Some approaches take the view that it is sufficient for the preconditions of the action to hold at the start state and it does not matter whether they hold while the action is in progress \citep{Knoblock1994}, whereas others hold the view that the preconditions of an action should be satisfied while the action is in progress \citep{Blum1997}. Moreover, some planning languages, such as Planning Domain Description Language (PDDL) \citep{Garrido2002,Fox2003},  distinguish between preconditions and those conditions that have to hold while the action is in progress. The latter conditions are referred to as \emph{invariant conditions}. Having invariant conditions different from preconditions undoubtedly brings more expressiveness to the planning language; however it comes at the price of higher implementation complexity. In this paper we take the position that the invariant conditions are the same as preconditions, which implies that the preconditions have to be preserved throughout the execution of the action.

The postconditions of a durative action cause changes to the state $s$ in which the action ends. These changes are: adding the positive postconditions $ps(a)^{+}$ to $s$ and deleting the negative postconditions  $ps(a)^{-}$ from $s$. Thus, for a state $s$ in which action $a$ ends, we have: $s \models ps(a)^+$ and $s \not\models ps(a)^-$.

\begin{example} \label{examact}
Action ``$\mathit{buildShelter}$" is available to the agents in the scenario described in Section~\ref{sec:scenario}. To build a shelter the agent has to secure and evacuate the area and there is no Shock detected -- these are represented as preconditions $\mathit{areaSecured}$,  $\mathit{evacuated}$ and $\neg \mathit{ShockDetected}$, respectively. Once the shelter is built the area does not have to remain evacuated. This is represented as the positive postcondition $\mathit{shelterBuilt}$, while the negative postcondition of this action is  $\neg \mathit{evacuated}$. In our scenario, we model this action as taking 4 units of time (that is, if it is executed in state $s_{k}$, it will end in state $s_{k+4}$,). 
\[
\mathit{buildShelter} = 
\left\langle
\left\{
\begin{array}{c}
\mathit{areaSecured}, \\
\mathit{evacuated}, \\
\neg \mathit{ShockDetected}
\end{array}
\right\},
\left\{
\begin{array}{c}
\mathit{shelterBuilt}, \\
\neg \mathit{evacuated} \\
\end{array}
\right\}, 4
\right\rangle
\]

\begin{figure}[!h]
  \centering
\tikzstyle{place}=[circle,draw, inner sep=0pt,minimum size=10mm]{
\begin{tikzpicture}
\node at  (2,-4)[place] [label={[align=center]below:$areaSecured$ \\ $evacuated$}] (3) {$s_{k}$} ;
\node at  (8,-4)[place] [label={[align=center]below:$shelterBuilt$}] (4) {$s_{k+4}$};
\draw[-latex, dashed]  (3) edge [above]node{$buildShelter$} (4);
\end{tikzpicture}}
\end{figure}
\end{example}

\subsection{Goals} \label{goal}
Goals identify the state of affairs in the world that an agent wants to achieve.
Different types of goals and their characteristics have been classified in the literature \citep{Riemsdijk2008}. 
\emph{Achievement\/} goals are the most common type of goals modelled in the agent literature and have therefore received the most attention \citep{Riemsdijk2008, Boer2002, Nigam2006, Riemsdijk2002}. goals for the purpose of this research are  achievement goals.

We define below the elements of the set $G$ of $P = (\mathit{F\!L}, \Delta, A, G, N )$. 

\begin{definition}[Goals]
A goal $g\in G$ is the pair $\langle r,v \rangle$, where $r$ is a possibly empty and finite set of well-defined literals representing the goal requirements, and $v\in\mathbb{N}, v> 0$, represents the utility/gain for achieving the goal.  
\end{definition}
Goal $g$'s requirements and value are denoted as $r(g)$ and $v(g)$, respectively.


\begin{example}\label{goals}
The goals from our illustrative scenario are formulated as below:
\[
\mathit{runningHospital} = 
\left\langle 
\left\{
\begin{array}{c}
\mathit{medicsPresent}, \\
\mathit{waterSupplied}, \\
\mathit{medicinesSupplied}
\end{array}
\right\},
25 
\right\rangle
\]
\[
\mathit{organiseSurvivorCamp} =
\left\langle 
\left\{
\begin{array}{c}
\mathit{areaSecured}, \\
\mathit{shelterBuilt}
\end{array}
\right\},
18
\right\rangle
\]
\end{example}
\subsection{Norms} \label{norm}

In this section we specify what we refer to as a norm in this work. 
In order to provide a context for the norm specification we explain how our norm specification corresponds to the five elements identified by \citet{Criado2012} that distinguish norm specification languages.

\begin{compactenum}
\item Deontic Operators: We model a permissive society in which the agent has complete knowledge of the domain of actions available. Everything is permitted unless it is explicitly prohibited. The role of obligation is to motivate the agent to execute a specific action and the role of prohibition is to inhibit the agent from executing a particular action. 

  \item Controls: Controls determine whether the deontic propositions operate on actions, states or both. In this work we  focus on action-based norms. 
  
\item Enforcement Mechanisms: We use the enforcement mechanism proposed by
\cite{Shams2015b} that is a combination of utility-based (e.g.,  \cite{Wamberto2011,Sofia2012}) and pressure-based \citep{Lopez2005} compliance methods.


\item Conditional Expressions: Similar to the control element, we use actions as conditional expressions. In other words, the norm condition is an action that once executed, the agent is obliged to or prohibited from executing the action that the norm targets. 

\item Temporal Constraints: temporal constraints can be used to express norm activation, termination, deadline, etc. The temporal constraint we specify here is concerned with the deadline. The agent is expected to comply with an obligation (execute a certain action) or a prohibition (refrain from executing a specific action) before some deadline. 
\end{compactenum} 

Having explained the properties of our norm specification language, we now define the element $N$ of problem $P = ( F\!L, \Delta, A, G, N )$. $N$ denotes a set of conditional norms to which the agent is subject:
\begin{definition}[Norms]\label{def:norms}
$N$ is a set of norms, each of which is a tuple of the form $n = \langle d\_o, a_{con}, a_{sub}, dl, c \rangle$, where 
\begin{compactitem}
\item $d\_o \in \{o,f\}$\footnote{The symbols $o$ and $f$ are normally represented as respectively $\mbox{\bf O}$ and $\mbox{\bf F}$ in the Deontic logic literature. However we have used lower case letters to make these consistent with our implementation in the next section. Capital letters in the implementation language are reserved for variables.} is the deontic operator determining the type of norm, which can be an obligation or a prohibition;
\item $a_{con} \in A$ is the durative action (cf. Def. \ref{acts}) that activates the norm;
\item $a_{sub} \in A$ is the durative action (cf. Def. \ref{acts}) that is the target of the obligation or prohibition; 
\item $dl \in \mathbb{N}$ is the norm deadline relative to the activation condition, which is the completion of the execution of the action $a_{con}$; and
\item $c \in \mathbb{N}$ is the penalty cost that will be applied if the norm is violated. $c(n)$ denotes the penalty cost of norm $n$.
 \end{compactitem}
\end{definition}
An obligation norm states that executing action $a_{con}$ obliges the agent to start/start and end the execution of $a_{sub}$ within $dl$ time units of the end of execution of $a_{con}$. Such an obligation is complied with if the agent starts or starts and ends executing $a_{sub}$ before the deadline and is violated otherwise. A prohibition norm expresses that executing action $a_{con}$ prohibits the agent from starting or starting and ending the execution of $a_{sub}$ within $dl$ time units of the end of execution of $a_{con}$. Such a prohibition is complied with if the agent does not start or does not start and end executing $a_{sub}$ before the deadline and is violated otherwise.

\begin{example}\label{normref}
The norms in the illustrative scenario are formulated as below:
\center
$n_{1} = \langle f, detectShock, buildShelter, 3, 5 \rangle$\\
$n_{2} = \langle o, detectPoison, stopWater,2, 10\rangle$
\end{example}

A norm can be activated multiple times in a sequence of action, generating different instances of the original norm. To make sure different instances are dealt with uniquely, we define \emph{instantiated norms}. In each instance the deadline is  updated relative to the end of execution of the action that is the condition of the norm. 

\begin{definition}[Instantiated Norm]\label{insno}
An instantiation of norm $n = \langle d\_o, a_{con},\allowbreak a_{sub}, dl, c \rangle$ is denoted as $n_{ins} = \langle d\_o, a_{con},a_{sub}, dl_{ins}, c \rangle$ where $dl_{ins}=dl + t_{a_{con}} + d_{a_{con}}$. $t_{a_{con}}$ is when action $a_{con}$ is executed and $d_{a_{con}}$ is the duration of $a_{con}$.
\end{definition}
We also denote an instantiation of a norm $n_i$ as $n_{i}'$.

\begin{example}
Assume that in some sequence of action $\mathit{detectShock}$ is executed at time $3$  (i.e. $t_{a_{con}}=3$) and that the duration of this action is $1$ (i.e. $d_{a_{con}}=1$). The instantiated version of norm 
\[
n_{1} = \langle f, \mathit{detectShock}, \mathit{buildShelter}, 3, 5 \rangle
\]
in this sequence of actions is 
\[
n'_{1} = \langle f, \mathit{detectShock}, \mathit{buildShelter}, 7,5\rangle
\]
Where $dl_{ins}$ is calculated based on Def.~\ref{insno}
\end{example}
\subsection{Semantics} \label{semantics}
Having explained the syntax of the model, we now focus on the semantics. To this end, we need to describe given a normative practical reasoning problem $P = (\mathit{FL}, \Delta, A, G, N )$: 
\begin{compactenum}[(i)] 
\item What the possible courses of action for the agent are and what 
properties each course of action has. Properties are defined in terms of goals that a sequence of action satisfies, norms it complies with and the norms it violates. This item is discussed in Section~\ref{propSeq}.
\item What different type of conflicts the agent can experience while trying to satisfy its goals and comply with the norms to which it is subject. In Section~\ref{conflict} we explore this item. 
\item What identifies a sequence of actions as a solution/plan for problem $P$. Plans are defined in Section~\ref{plan}.
\end{compactenum}

\subsubsection{Sequences of Actions and their Properties} \label{propSeq}

Let $P=(\mathit{FL},\Delta,A,G,N)$ be a normative practical reasoning problem. Also let $\pi =\langle (a_{0},0),\cdots,$ $(a_{n},t_{a_{n}}) \rangle$ with $a_{i} \in A$ and $t_{a_{i}} \in \mathbb{Z^{+}}$ be a sequence of actions $a_{i}$ executed at time $t_{a_{i}}$. The pair $(a_{i},t_{a_{i}})$ reads as action $a_{i}$ is executed at time $t_{a_{i}} \in \mathbb{Z^{+}} \mbox{ s.t. } \forall i<j, t_{a_{i}} < t_{a_{j}}$. The total duration of a sequence of actions, $\mathit{Makespan}(\pi)$, is defined in Equation \ref{makespan}. 
\begin{equation}\label{makespan}
 \mathit{Makespan}(\pi) = \mathit{max}(t_{a_{i}} + d_{a_{i}}) 
\end{equation}
Actions in a sequence can be executed concurrently but they cannot start at the same time. This is because the planning problem is defined for a single agent and a single agent is not typically assumed to be able to start two actions at the exact same instant. Also actions in the sequence should not have concurrency conflicts, which are defined below.
In our presentation we need to check for the occurrence of specific pairs in a sequence of actions $\pi$, and we thus define the operator ``$\widehat{\in}$'' as
\[
(a,t_{a})\;\widehat{\in}\;\pi\mbox{ iff } 
\left\{
\begin{array}{l}
\pi=\langle (a,t_{a}),\ldots,(a_n,t_n)\rangle \mbox{ or } \\
\pi=\langle (a_0,0),\ldots,(a,t_{a}),\ldots,(a_n,t_n)\rangle \mbox{ or } \\
\pi=\langle (a_0,0),\ldots,(a,t_{a})\rangle
\end{array}
\right.
\]

Temporary conflicts prevent the agent from executing two actions under specific constraints, the most common one of which is time. Conflicts caused by time, known as \emph{concurrency conflicts\/} between actions, prevent actions from being executed in an overlapping period of time. \citet{Blum1997} define that two actions $a_{i}$ and $a_{j}$ cannot be executed concurrently, if at least one of the following holds:
\begin{compactenum}
  \item The preconditions of $a_{i}$ and $a_{j}$  contradict each other:
        \[\exists r \in pr(a_{i}) \mbox{ s.t. } \neg r \in pr(a_{j}) \; \; \; or \]
        \[\exists \neg r \in pr(a_{i}) \mbox{ s.t. } r \in pr(a_{j})\]
  \item The postconditions of $a_{i}$ and $a_{j}$  contradict each other:
        \[\exists r \in ps(a_{i})^{+} \mbox{ s.t. } \neg r \in ps(a_{j})^{-} \; \; \; or \]
        \[\exists \neg r \in ps(a_{i})^{-} \mbox{ s.t. } r \in ps(a_{j})^{+}\]
        
  \item The postconditions of $a_{i}$  contradict the preconditions of 
  $a_{j}$:
         \[\exists r \in ps(a_{i})^{+} \mbox{ s.t. } \neg r \in pr(a_{j}) \; \; \; or \]
         \[\exists \neg r \in ps(a_{i})^{-} \mbox{ s.t. } r \in pr(a_{j})\]
         
  \item The preconditions of $a_{i}$ are contradicted by the postconditions of 
  $a_{j}$:
         \[\exists r \in pr(a_{i}) \mbox{ s.t. } \neg r \in ps(a_{j})^{-} \; \; \; or \]
         \[\exists \neg r \in pr(a_{i}) \mbox{ s.t. } r \in ps(a_{j})^{+}\]
\end{compactenum}
We summarise the four conditions above in Definition \ref{conConf}, where we define 
what are referred to as conflicting actions in the remainder of this work.
\begin{definition}[Conflicting Actions]\label{conConf}
Actions $a_{i}$ and $a_{j}$ have a concurrency conflict if the pre- or post-conditions of $a_{i}$ contradict the pre- or post-conditions of $a_{j}$. The set of conflicting actions is denoted as $\mathit{cf}_{\mathit{action}}$:
\begin{multline}
\mathit{cf}_{\mathit{action}}=
\left\{
(a_{i},a_{j})  
\left|
\begin{array}{l}
  \exists r \in pr(a_{i}) \cup  ps(a_{i})^{+}, \neg r \in pr(a_{j}) \cup 
  ps(a_{j})^{-} \\
  \mbox{ or } \\
  \exists \neg r \in pr(a_{i}) \cup  ps(a_{i})^{-}, r \in pr(a_{j}) \cup 
  ps(a_{j})^{+}
  \end{array}
  \right.
  \right\}
\end{multline}
\end{definition}

\begin{example} \label{examConfAct}
Assume action $evacuate$ with the following specification:
\[
\mathit{evacuate} = 
\left\langle
\left\{
\begin{array}{c}
\mathit{populated}, \\
\mathit{ShockDetected}, 
\end{array}
\right\},
\left\{
\begin{array}{c}
\mathit{evacuated}, \\
\neg \mathit{populated} \\
\end{array}
\right\}, 5
\right\rangle
\]
The pre- and post-conditions of this action are inconsistent with the pre- and post-conditions of action $\mathit{buildShelter}$ defined in Example~\ref{examact}:
\[
\mathit{buildShelter} = 
\left\langle
\left\{
\begin{array}{c}
\mathit{areaSecured}, \\
\mathit{evacuated}, \\
\neg \mathit{ShockDetected}
\end{array}
\right\},
\left\{
\begin{array}{c}
\mathit{shelterBuilt}, \\
\neg \mathit{evacuated} \\
\end{array}
\right\}, 4
\right\rangle
\]
Therefore, these two actions cannot be executed concurrently. However, action $\mathit{evacuate}$  effectively contributes to the preconditions of $\mathit{buildShelter}$, which means they can indeed be executed consecutively.
\end{example}

\begin{definition}[Sequence of States]
Let $\pi =\langle (a_{0},0),\cdots,$ $(a_{n},t_{a_{n}}) \rangle$ be a sequence of actions such that  $\nexists (a_{i},t_{a_{i}}), (a_{j},t_{a_{j}}) \in \pi \mbox{ s.t. } t_{a_{i}} \leq t_{a_{j}} < t_{a_{i}}+d_{a_{i}}, (a_{i}, a_{j}) \in \mathit{cf}_{action}$ and let $m= \mathit{Makespan}(\pi)$. The execution of a sequence of actions $\pi$ from a given starting state $s_{0}= \Delta$ brings about a sequence of states $S(\pi)=\langle s_{0}, \cdots s_{m}\rangle$ for every discrete time interval from $0$ to $m$.
\end{definition}

The transition relation between states is given by Def.~\ref{trans}. If  action $a_{i}$ ends at time $k$, state $s_{k}$ results from removing delete post-conditions and adding add post-conditions of action $a_{i}$ to state $s_{k-1}$. If there is no action ending at $s_{k}$, the state $s_{k}$ remains the same as $s_{k-1}$. We first define $A_{k}$ as the set of action/time pairs such that the actions end at some specific state $s_{k}$:
\begin{equation}
A_{k} = \{ (a_{i}, t_{a_{i}}) \in \pi \,|\, k=t_{a_{i}}+d_{a_{i}}\}
\end{equation}
Note that $s_{k}$ is always well-defined since two actions with inconsistent post-conditions, according to Def.~\ref{conConf} belong to $\mathit{cf}_{action}$ so they cannot be executed concurrently and thus they cannot end at the same state.

\begin{definition}[State Transition]\label{trans}
Let $\pi =\langle (a_{0},0),\cdots,$ $(a_{n},t_{a_{n}}) \rangle$ and let $S(\pi)=\langle s_{0}, \cdots s_{m}\rangle$ be the sequence of states brought about by $\pi$:
\begin{equation}\label{transition}
 \forall k > 0 :
   \; s_{k} = 
   \begin{cases}  
    (s_{k-1} \setminus (\bigcup\limits_{a \in A_{k}} ps(a)^{-}) \cup 
    \bigcup\limits_{a \in A_{k}} ps(a_{i})^{+} & A_{k} \neq  \emptyset\\
     s_{k-1} & A_{k} = \emptyset
   \end{cases}
\end{equation}
\end{definition}
We now turn our attention to the properties of each sequence of actions.

\begin{definition}[Goal Satisfaction]
Goal requirements should hold in order to satisfy the goal. A sequence of actions $\pi =\langle (a_{0},0),\cdots$ $(a_{n},t_{a_{n}}) \rangle$ satisfies goal $g$ if there is at least one state $s_k \in S(\pi)$ that satisfies the goal:
\begin{equation}
\pi \models r(g) \mbox{ iff } \exists \; s_{k} \in S(\pi)  \mbox{ s.t. } s_{k} \models r(g)
\end{equation}  
\end{definition}
The set of goals satisfied by $\pi$ is denoted as $G_{\pi}$:
\begin{equation} \label{G'}
  G_{\pi}=\{g \,|\, \pi \models r(g)\}
\end{equation}

\begin{definition}[Activated Norms]\label{activeNorm}
 A norm $n=\langle d_o, a_{con}, a_{sub},dl, c \rangle$ is instantiated in a sequence of actions $\pi =\langle (a_{0},0),\cdots,(a_{n},t_{a_{n}}) \rangle$ if its activation condition $a_{con}$ belongs to the sequence of actions. Let $N_{\pi}$ be the set of instantiations of various norms in $\pi$ defined in Equation \ref{eq:act}. Note that $dl_{ins}$ is calculated based on Definition \ref{insno}.
\begin{equation}\label{eq:act}
 N_{\pi} = 
 \{\langle d_o, a_{\mathit{con}}, a_{\mathit{sub}},\mathit{dl}_{\mathit{ins}},c\rangle
 \;
 |\;
    \langle d_o, a_{con}, a_{sub},dl,c\rangle \in N, (a_{con},t_{a_{con}}) \;\widehat{\in}\;\pi
    \}
\end{equation}
 \end{definition}

\begin{definition}[Obligation Compliance]\label{oc}
A sequence of actions $\pi =\langle (a_{0},0),$ $\cdots,(a_{n},t_{a_{n}}) \rangle$ complies with an obligation $n = \langle o, a_{\mathit{con}}, a_{\mathit{sub}},\mathit{dl}_{\mathit{ins}}, c \rangle$ if $a_{\mathit{con}}$ is executed in $\pi$ and $a_{\mathit{sub}}$, starts (cf. Eq.~\ref{oblcompst}) or starts and ends (cf. Eq.~\ref{oblcompen}) within the period when the condition holds and when the deadline expires:
\begin{equation}\label{oblcompst}
  \pi \models n \mbox{ iff } 
 (a_{con},t_{a_{con}}),(a_{sub},t_{a_{sub}})\;\widehat{\in}\;\pi \mbox{ s.t. } 
  t_{a_{\mathit{sub}}} \in [t_{a_{\mathit{con}}}+d_{a_{\mathit{con}}}, \mathit{dl}_{\mathit{ins}})
 \end{equation}
\begin{multline}\label{oblcompen}
  \pi \models n \mbox{ iff } 
  (a_{con},t_{a_{con}}),(a_{sub},t_{a_{sub}}) \;\widehat{\in}\;\pi \mbox{ s.t. } \\
[t_{a_{sub}},t_{a_{sub}}+ d_{a_{sub}}] \subseteq [t_{a_{con}}+d_{a_{con}}, dl_{ins})
\end{multline}
\end{definition}

\begin{definition}[Obligation Violation]\label{ov}
A sequence of actions $\pi =\langle (a_{0},0),$ $\cdots,(a_{n},t_{a_{n}}) \rangle$ violates obligation  $n_{ins} = \langle o, a_{con}, a_{sub},dl_{ins}, c \rangle$ if $a_{con}$ is executed in $\pi$, but $a_{sub}$ does not start (Equation \ref{oblvolst}), or does not start and end (Equation \ref{oblvolen}) in the period between the state when the condition holds and when the deadline expires.
\begin{multline}\label{oblvolst}
  \pi \not\models n \mbox{ iff } (a_{con},t_{a_{con}}) \;\widehat{\in}\; \pi,  
  (a_{sub},t_{a_{sub}}) \;\widehat{\not\in}\; \pi \mbox{ s.t. } \\ t_{a_{sub}} \in [t_{a_{con}}+d_{a_{con}}, 
  dl_{ins})
\end{multline}
\begin{multline}\label{oblvolen}
  \pi \not\models n \mbox{ iff } (a_{con},t_{a_{con}}) \;\widehat{\in}\; \pi, 
  (a_{sub},t_{a_{sub}}) \;\widehat{\not\in}\; \pi \mbox{ s.t. } \\
[t_{a_{sub}},t_{a_{sub}}+ d_{a_{sub}}] \subseteq [t_{a_{con}}+d_{a_{con}}, dl_{ins} )
\end{multline}
\end{definition}

\begin{example}
Let there be the following instantiated version
\[
n'_{2} = \langle o, \mathit{detectPoison},\mathit{stopWater},8, 10\rangle
\]
of norm 
\[
n_{2}= \langle o, \mathit{detectPoison},\mathit{stopWater},2, 10\rangle
\]
from Example~\ref{normref}. The compliance period for this norm in displayed in the figure below. According to Def.~\ref{oc} in its Eq.~\ref{oblcompst}, if $t_{\mathit{stopWater}}$ belongs to this period, this norm instance is complied with; otherwise, according to Def.~\ref{ov} in its Eq.~\ref{oblvolst}, the norm is violated. This is illustrated by the following diagram:
\begin{figure}[!h]
  \centering
\tikzstyle{place}=[circle,draw, inner sep=0pt,minimum size=10mm]{
\begin{tikzpicture}
\node at  (0,-4)[place] (3) {5} ;
\node at  (4,-4)[place] (4) {6}; 
\node at  (10,-4)[place]  (5) {8};
\draw[-latex, dashed]  (3) edge [above]node{$\mathit{detectPoison}$} (4);
\draw[latex-latex,red]  (4) edge [above]node{compliance period} (5);
\end{tikzpicture}}
\end{figure}
\end{example}

\newpage
\begin{example}
Let there be the following instantiated version
\[
n'_{3} = \langle o,\mathit{detectEarthquake},\mathit{blockMainRoad},7, 12 \rangle
\]
of norm 
\[
n_{3}=\langle o,\mathit{detectEarthquake}, \mathit{blockMainRoad},5, 12\rangle
\]
which obliges the agent to have blocked the main road within 5 units of time after detecting an earthquake. Since the post-conditions of action $\mathit{blockMainRoad}$ are brought about at the end of its execution, according to Def.~\ref{oc} (Eq.~\ref{oblcompen}), this norm is complied with if $\mathit{blockMainRoad}$ starts and ends between time points 2 and 7. Otherwise, according to Def.~\ref{ov} (Eq.~\ref{oblvolen}) this norm is violated. 
\end{example}

\begin{definition}[Prohibition Compliance]\label{pc}
A sequence of actions $\pi =\langle (a_{0},0),$ $\cdots,(a_{n},t_{a_{n}}) \rangle$ complies with prohibition $n = \langle f, a_{con}, a_{sub},dl_{ins}, c \rangle$ if  $a_{con}$, is executed and $a_{sub}$, does not start (Eq.~\ref{procompst}) or does not start and end (Eq.~\ref{procompen}) in the period when the condition holds and the deadline expires. Formally:
\begin{multline}\label{procompst}
  \pi \models n \mbox{ iff } (a_{con},t_{a_{con}}) \;\widehat{\in}\; \pi, 
  (a_{sub},t_{a_{sub}}) \;\widehat{\not\in}\; \pi \mbox{ s.t. } \\
   t_{a_{sub}} \in [t_{a_{con}}+d_{a_{con}}, dl_{ins})
\end{multline}
\begin{multline}\label{procompen}
  \pi \models n \mbox{ iff } (a_{con},t_{a_{con}}) \;\widehat{\in}\; \pi,
  (a_{sub},t_{a_{sub}}) \;\widehat{\not\in}\; \pi \mbox{ s.t. } \\
   [t_{a_{sub}},t_{a_{sub}}+d_{a_{sub}}] \subseteq [t_{a_{con}}+d_{a_{con}}, dl_{ins})
\end{multline}
\end{definition}

\begin{definition}[Prohibition Violation]\label{pv}
A sequence of actions $\pi =\langle (a_{0},0),$ $\cdots,(a_{n},t_{a_{n}}) \rangle$ violates prohibition $n = \langle f, a_{con}, a_{sub},dl_{ins}, c \rangle$ iff $a_{con}$, has occurred and $a_{sub}$ starts (Eq.~\ref{provolst}) or  starts and ends (Eq.~\ref{provolen}) in the period between when the condition holds and when the deadline expires. Formally:
\begin{equation}\label{provolst}
  \pi \not\models n \mbox{ iff } (a_{con},t_{a_{con}}), (a_{sub},t_{a_{sub}}) \;\widehat{\in}\; \pi 
  \mbox{ s.t. }  t_{a_{sub}} \in [t_{a_{con}}+d_{a_{con}}, dl_{ins})
\end{equation}
\begin{multline}\label{provolen}
  \pi \not\models n \mbox{ iff } (a_{con},t_{a_{con}}), (a_{sub},t_{a_{sub}}) \;\widehat{\in}\; \pi 
  \mbox{ s.t. } \\    [t_{a_{sub}},t_{a_{sub}}+d_{a_{sub}}] \subseteq [t_{a_{con}}+d_{a_{con}}, dl_{ins})
\end{multline}
\end{definition}

\begin{example}
Let there be the following instantiated version
\[
n'_{1} = \langle f, \mathit{detectShock}, \mathit{buildShelter}, 6, 5 \rangle
\]
of norm 
\[
n_{1}= \langle f, \mathit{detectShock}, \mathit{buildShelter}, 3, 5 \rangle
\] 
presented in Ex.~\ref{normref}. The compliance period for this norm in displayed in the figure below. According to Def.~\ref{pv} (Eq.~\ref{provolst}), if $t_{\mathit{buildShelter}}$ belongs to this period, this norm instance is violated; otherwise, according to Def.~\ref{pc} (Eq.~\ref{procompst}), it is complied  with.
\end{example}

\begin{example}
Let there be the following instantiated version
\[
n'_{4} = \langle f,\mathit{detectEarthquake},\mathit{blockMainRoad},7, 12 \rangle
\]
of norm 
\[
n_{4}=\langle f,\mathit{detectEarthquake}, \mathit{blockMainRoad},5, 12\rangle
\]
which forbids the agent from blocking the main road within 5 units of time after detecting an earthquake. Since the post-conditions of action $\mathit{blockMainRoad}$ are brought about at the end of its execution, according to Def.~\ref{pc} (Eq.~\ref{procompen}), this norm is violated if $\mathit{blockMainRoad}$ starts and ends between time points 2 and 7. Otherwise, according to Def.~\ref{pv} (Eq.~\ref{provolen}) this norm is complied with. This is illustrated by the diagram below.
\end{example}



The set of norms complied with and violated in $\pi$ are denoted as $N_{cmp(\pi)}$
and $N_{vol(\pi)}$ respectively, and defined as follows:
\begin{equation}
  N_{cmp(\pi)}=\{n_{ins} \in N_{\pi} \;|\; \pi \models n_{ins} \}
\end{equation}
\begin{equation}
  N_{vol(\pi)}=\{n_{ins} \in  N_{\pi} \;|\; \pi \not\models n_{ins} \}
\end{equation}
To make sure there are no norms pending at $m=\mathit{Makespan}(\pi)$, we assume that the norm deadlines are smaller than $m$. Therefore, all the activated norms in $\pi$ are either complied with or violated by time $m$:
\begin{equation}
N_{\pi}=N_{cmp(\pi)} \cup N_{vol(\pi)}
\end{equation}
Alternatively, it could be assumed that the norms that are pending (i.e. neither violated nor complied with) at $m$ are all considered as complied with or violated.

\subsubsection{Conflict} \label{conflict}

In the previous section we defined when a sequence of actions satisfies a goal, complies with or violates a norm. A possible cause for not satisfying a certain goal is the conflict between the goal and another goal or norm. Likewise, violating a norm could help in reaching a goal or complying with another norm. In this work we do not concern ourselves directly with detecting or resolving  conflicts, instead, we focus on the consequences of such conflicts on the agent behaviour. To make this work self-contained, however, we briefly review the causes of conflicts between goals, between norms and between goals and norms. We leave for future work the provision of agents with the capability to reason about the plans and consequently \emph{inferring\/} the conflict between goals, between norms and between goals and norms.

An agent may pursue multiple goals or desires at the same time and it is likely that some of these goals conflict \citep{Riemsdijk2002, Nigam2006, Pokahr2005, Thangarajah2003, Riemsdijk2009}. 
Conflict between the agent's goals or desires, especially for BDI agents, has been addressed by several authors. \citet{leendert2004} describe two goals as conflicting if achieving them requires taking two conflicting actions, where conflicting actions are encoded using integrity constraints. \citet{Rahwan2006} on the other hand, define two desires as conflicting if the sets of beliefs that supports the achievement of desires are contradictory. Like \citet{Rahwan2006}, \citet{Broersen2002} argue that for a set of goals not to be conflicting, a consistent mental attitude (e.g. beliefs and norms) is required. Some (\emph{e.g.}, \citep{Toniolo2013}) have adopted a static view on goal conflict, in which conflicting goals are mutually-exclusive, hence impossible to satisfy in the same plan regardless of the order or choice of actions in the plan. Limited and bounded resources (\emph{e.g.} time, budget, etc.) are debated as another cause of conflict between goals \citep{Thangarajah2002}.

\citet{Lopez2005} discuss conflict between goals and norms in terms of goals being hindered by norms or vice-versa.
The same applies to the approach offered by \citet{Modgil2008}, suggesting a mechanism to resolve the conflicts between desires and normative goals. In this approach, norms are represented as system goals that may conflict with an agent's goals or desires. Social goals and individual goals do not need to conflict directly. Instead, conflict arises from the reward or punishment of complying with or violating a norm that may facilitate or hinder some of the agent's individual goals.
\cite{Shams2015b} identify the conflict between norms and goals as follows.
When an obligation forces the agent to take an action that has postconditions that are inconsistent with the requirements of a goal, they may come into conflict. On the other hand, when an action is prohibited, and the postconditions of that action contribute to a goal, they may conflict.

Conflict between norms have been studied in multi-agent systems (\emph{e.g.}, \citet{Norman2009}) as well as other domains such as legal reasoning (\emph{e.g.}, \citep{Sartor1992}). When faced with conflicting norms, the agent cannot comply with both of them and hence one of the norms is violated. In terms of action-based norms, \cite{Shams2015b} define two obligations conflicting if they oblige the agent to take two conflicting actions (cf. Def. \ref{conConf}) in an overlapping time interval. Likewise, an obligation and a prohibition cause conflict if they oblige and forbid the agent to execute the same action in an overlapping time interval.

Having defined sequences of actions and the properties and conflicts they can 
involve, we can now define which sequences of action can be identified as plans in the next section.
\subsubsection{Plans}\label{plan}
In classical planning a sequence of actions $\pi=\langle (a_{0},0),\cdots,(a_{n},t_{n})\rangle$ is identified as a plan if all the fluents in the initial state, do hold at time $0$ and for each $i$, the preconditions of action $a_{i}$ hold at time $t_{a_{i}}$, and the goal of planning problem is satisfied in time $m$, where $m = \mathit{Makespan}(\pi)$. 
However, extending the conventional planning problem by multiple potentially 
conflicting goals and norms requires defining extra conditions in order to make a 
sequence of actions a plan and a solution for $P$. In what follows, 
we define what is required to identify a sequence of actions as a plan.

\begin{definition}[Plan]\label{plandef}
A sequence of actions $\pi =\langle (a_{0},0),$ $\ldots, (a_{n},t_{a_{n}}) \rangle \mbox{ s.t. } \nexists$ $(a_{i},t_{a_{i}}), (a_{j},t_{a_{j}}) \;\widehat{\in}\; \pi \mbox{ s.t. } t_{a_{i}}  \leq t_{a_{j}} < t_{a_{i}}+d_{a_{i}}, (a_{i}, a_{j}) \in \mathit{cf}_{action}$ is a plan for the normative practical reasoning problem $P = ( F\!L, \Delta, A, G, N )$ 
if the following conditions hold: 
\begin{compactitem}
 \item fluents in $\Delta$ (and only those fluents) hold in the initial state: $s_{0} = \Delta $
 \item the preconditions of action $a_{i}$ holds at time $t_{a_{i}}$ 
 and throughout the execution of $a_{i}$:
 $\forall k \in [t_{a_{i}}, t_{a_{i}}+d_{a_{i}}), s_{k} \models pr(a_{i})$
 \item plan $\pi$ satisfies a non-empty subset of goals:
$G_{\pi} \neq \emptyset$
  \end{compactitem}
\end{definition}

The utility of a plan $\pi$ is defined by deducting the penalty costs of violated norms from the value gain of satisfying goals (Equation~\ref{utility}). The set of optimal plans, $\mathit{Opt}$, are those plans that maximise the utility.
\begin{equation}\label{utility}
\mathit{Utility}(\pi) = \sum_{g_{i} \in G_{\pi}} v(g_{i}) - \sum_{n_{j} \in N_{vol(\pi)}} c(n_{j})
\end{equation}
Examples of calculating the utility of plans are in \ref{LpScen}.

The set $\mathit{Opt}$ is empty only if there are no plans for the planning problem. Otherwise, the utility function is guaranteed to terminate and find the optimal plans and hence populate the set $Opt$.

\section{Implementation}
\label{sec:implementation}
In this section, we demonstrate how a normative practical reasoning problem $P=(\mathit{FL}, \Delta, A,$  $G, N)$ (cf. Def. \ref{NPS}), can be implemented. Our implementation should be seen as a proof of concept that  provides a computational realisation of all aspects of the formal model. We use Answer Set Programming (ASP)  \citep{Gelfond1988} to propose such an implementation. Recent work on planning in ASP \citep{aspplanning} demonstrates that in terms of general planners ASP is a viable competitor. The Event Calculus (EC) \citep{Kowalski1986} forms the basis for the implementation of some normative reasoning frameworks, such as those of \cite{Alrawagfeh2014} and \cite{Artikis2009}.  Our proposed formal model is independent of language and could be translated to EC and hence to a computational model. However, the one-step translation to an ASP is preferred because the formulation of the problem is much closer to a computational model, thus there is a much narrower conceptual gap to bridge.  Furthermore, the EC implementation language is Prolog, which although syntactically similar to ASP, suffers from non-declarative features such as clause ordering affecting the outcome and the cut (``!'') operator, jeopardising its completeness.  Also, its query-based nature which focuses on one query at a time, makes it cumbersome to reason about all plans.

In what follows, we provide a brief introduction to ASP in Section \ref{(ASP)}, followed by the mapping of normative practical reasoning problem $P=(\mathit{FL}, \Delta, A,$  $G, N)$ (cf. Def. \ref{NPS}) into ASP in Section \ref{Trans}. In the latter section we show how $P$ is mapped into an answer set program such that there is a one to one correspondence between solutions for the problem and the answer sets of the program. The mapping itself is provided in Figure \ref{mapping}. The explanation of the mapping is presented in Sections \ref{st}--\ref{subsec:opt}
, with cross references to the code fragments listed in Figure \ref{mapping}.

\subsection{Answer Set Programming}\label{(ASP)}
ASP is a declarative programming paradigm using logic programs under Answer Set semantics \citep{Lifschitz2008}. Like all declarative paradigms it has the advantage of describing the constraints and the solutions rather than the writing of an algorithm to find solutions to a problem. A variety of programming languages for ASP exist, and we use \AnsProlog \citep{Baral2003}. There are several efficient solvers for \AnsPrologs, of which \textsc{Clingo} \citep{gekakaosscsc11a}  and \textsc{DLV} \citep{Eiter99thediagnosis} are currently the most widely used. 

The basic components of \AnsProlog are \emph{atoms} that are constructs to which one can assign a truth value. An atom can be negated, adopting \emph{negation as failure\/} (naf), which establishes that a negated atom $\notf{a}$ is true if there is no evidence to prove $a$. {\em Literals}
are atoms $a$ or negated atoms $\notf{a}$ (referred to as naf-literals).
Atoms and literals are used to create rules of the general form ``$a :\!\!-\;  b_{1}, ..., b_{m}, \notf{c_{1}}, ..., \notf{c_{n}}.$''
where $a,b_{i}$ and $c_{j}$ are atoms. Intuitively, a rule means that {\em if all atoms
$b_{i}$ are known/true and no atom $c_{j}$ is known/true, then $a$ must be
known/true}. We refer to $a$ as the \emph{head of the rule} and $b_{1}, ..., b_{m}, \notf{c_{1}},..., \notf{c_{n}}$ as the \emph{body of the rule}.  A rule with an empty body is called a {\em fact} and a rule with an empty head is called a {\em constraint}, indicating that no solution should be able to satisfy the body. 
Another type of rules are called choice rules and are denoted as 
$l\{h_{0}, \cdots, h_{k} \}u :-  \; l_{1}, \cdots, l_{m}, not \; $
$l_{m+1}, \cdots, not \; l_{n}.$, in which $h_{i}$s and $l_{i}$s are atoms. $l$ 
and $u$ are integers and the default values for them are 0 and 1, respectively.   
A choice rule is satisfied if the number of atoms belonging to 
$\{h_{0}, \cdots, h_{k} \}$ that are true/known is between the lower bound $l$ and 
upper bound $u$. A \emph{program} is a set of rules representing their conjunction. The semantics of \AnsProlog is defined in terms of \emph{answer sets}, i.e. assignments of true and false to all atoms in the program
that satisfy the rules in a minimal and consistent fashion. A program may have zero or more answer sets, each corresponding to a solution. We refer to \ref{proofComp} and \cite{Baral2003} for a formal treatment of the semantics of ASP.

\subsection{Translating the Normative Practical Reasoning Problem into ASP}\label{Trans}
Prior to describing the mapping (Figure \ref{mapping}) of normative practical reasoning  problem $P=(\mathit{FL}, \Delta, A,$  $G, N)$ into ASP, we list the atoms that we use in  the mapping in Table \ref{atm}. 
The description of the mapping is presented 
in the following sections with references to Figure \ref{mapping}. Note that variables are in capitals and grounded variables are in small italics.

\begin{table}[t]
\centering
\caption{Atoms and their Intended Meaning}
\begin{tabular}{|l|l|} \hline
{\bf Atom} & {\bf Intended meaning}\\ \hline
\asp{state(S)} &  \asp{S} is a state\\ \hline
\asp{holdsat(F,S)} & fluent \asp{F} holds in state \asp{S} \\ \hline
\asp{terminated(F,S)} & fluent \asp{F} terminates in state \asp{S}\\ \hline
\asp{action(A,D)} &  \asp{A} is an action with duration \asp{D} \\ \hline
\asp{executed(A,S)} & action \asp{A} is executed in state \asp{S} \\ \hline
\asp{pre(A,S)}  & the preconditions of action \asp{A} hold in state \asp{S} \\ \hline
\asp{inprog(A,S)} & action \asp{A} is in progress in state \asp{S} \\ \hline
\asp{goal(G,V)} &  \asp{G} is a goal with value \asp{V} \\ \hline
\asp{satisfied(G,S)} & goal \asp{G} is satisfied in state \asp{S}\\ \hline
\asp{norm(N,C)} &  \asp{N} is a norm with the violation cost of \asp{C}\\ \hline
\asp{cmp(o|f(N,S1,A,DL),S2)} &  \parbox[t]{8cm}{ the instantiated version of obligation or prohibition  \asp{N} in state \asp{S1}, is complied with in state \asp{S2}\vspace{0.5 mm}}\\ \hline
\asp{vol(o|f(N,S1,A,DL),S2)} & \parbox[t]{8cm}{ the instantiated version of obligation or prohibition  \asp{N} in state \asp{S1}, is violated  in state \asp{S2}}\\ \hline
\asp{value(TV)} & \asp{TV} is the sum of values of goals satisfied in a plan \\ \hline
\asp{cost(TC)} & \parbox[t]{8cm}{ \asp{TC} is the sum of violation costs of norms violated in a plan \vspace{0.5 mm}}\\ \hline
\asp{utility(U)} &  \parbox[t]{8cm}{ \asp{U} is the utility of a plan as the difference of values of goals satisfied and norms violated in the plan \vspace{0.5 mm}} \\ \hline
\end{tabular}
\label{atm}
\end{table}


\subsubsection{States} \label{st}
In Section \ref{semantics} we described the semantics $P=(\mathit{FL}, \Delta, A, G, N)$ over a set of states. 
The facts produced by line \ref{time1} provide the program with all available 
states for plans of maximum length $q$. 
Currently, the length of the plans needs to be determined experimentally. We plan to automate this using incremental features of ASP solver \textsc{clingo4}  \citep{gekakaosscsc11a}.  
Line \ref{initial} encodes that the initial fluents, ($x \in \Delta$) need to hold at state $0$ which is achieved by the facts 
$\texttt{holdsat($x,0$)}$.
Fluents are inertial, they continue to hold unless they are terminated. Inertia is encoded in lines \ref{ag101}--\ref{ag102}. Termination is modelled through the predicate $\texttt{terminated(X,S)}$.
\subsubsection{Actions} \label{ac}
This section describes the details of encoding of actions. Each durative action is encoded as $\texttt{action($a,d$)}$ (line \ref{ag1}), where $a$ is the name of the action and $d$ is its duration. The preconditions $pr(a)$ of action $a$ hold in state $s$ if $s \models pr(a)$. This is expressed in line \ref{ag2} using atom  $\texttt{pre($a$,S)}$
In order to make the coding more readable we introduce the shorthand $\texttt{EX(X,S)}$, where $\texttt{X}$ is a set of fluents that should hold at state $\texttt{S}$. For all $x \in \texttt{X}$,  $\texttt{EX(X,S)}$ is translated into $\texttt{holdsat($x$,S)}$ 
and for all $\neg x \in \texttt{X}$, $\texttt{EX(}$$\neg$$ \texttt{X,S)}$ is 
translated into $\texttt{not EX($x$,S)}$ using negation as failure.

The agent has the choice to execute any of its actions in any state. This is 
expressed in the choice rule in line \ref{ag3}. Since no lower or upper bound is given for $\texttt{\{executed(A,S)\}}$, the default value of $\texttt{0\{executed(A,S)\}1}$ is implied, meaning that the agent has the choice of whether or not to execute an action.  Following the approach in \cite{Blum1997}, we assume that the preconditions of a durative action should be preserved when it is in progress. We first encode the description of an action in progress, followed by ruling out the possibility of an action being in progress in the absence of its preconditions. A durative action is in progress, $\texttt{inprog(A,S)}$, from the state in which it starts up to the state in which it ends (lines \ref{ag41}--\ref{ag42}). Line \ref{ag4}, rules out the execution of an action, when the preconditions of the action do not hold during its execution. A further assumption made is that the agent cannot 
start two actions at exactly the same time (line \ref{ag51}--\ref{ag52}). Once an action starts, the result of its execution is reflected in the state where the action ends. This is expressed through
\begin{inparaenum}[(i)]
\item lines \ref{ag81}--\ref{ag82} that 
allow the add postconditions of the action to hold when the action ends, and 
\item line \ref{ag91}--\ref{ag92} that allow the termination 
of the delete postconditions.
\end{inparaenum}
Termination takes place in the state {\em before\/} the end state of the action: the reason for this is the inertia of fluents that was expressed in 
lines \ref{ag101}--\ref{ag102}. Thus delete post-conditions of an action are terminated in the state before the end state of the action, so that they will not hold in the following state, in which the action 
ends (i.e. they are deleted from the state).

\begin{figure*}
Creating states: $\forall \; k \in [0,q]$
\begin{lstlisting}[name=main,mathescape]
state($k$). $\label{time1}$
\end{lstlisting}
Setting up the initial state: $\forall \; x \in \Delta$
\begin{lstlisting}[name=main,mathescape] 
holdsat($x,0$). $\label{initial}$
\end{lstlisting}
Rule for fluent inertia
  \begin{lstlisting}[name=main,mathescape]
holdsat(X,S2) :- holdsat(X,S1), not terminated(X,S1), $\label{ag101}$
                 state(S1), state(S2), S2=S1+1. $\label{ag102}$
  \end{lstlisting}
  Creating the actions and their preconditions: $\forall a \in A$, $a=\langle pr, ps, d\rangle$ 
  \begin{lstlisting}[name=main,mathescape] 
action($a,d$). $\label{ag1}$
pre($a$,S) :- EX($pr(a)^{+}$,S), not EX($pr(a)^{-}$,S), state(S). $\label{ag2}$
  \end{lstlisting}
    Common constraints on action execution
  \begin{lstlisting}[name=main,mathescape]
{executed(A,S)} :- action(A,D), state(S). $\label{ag3}$
inprog(A,S2) :- executed(A,S1), action(A,D), $\label{ag41}$
                state(S1), state(S2), S1<=S2, S2<S1+D. $\label{ag42}$
:- inprog(A,S), action(A,D), state(S), not pre(A,S). $\label{ag4}$
:- executed(A1,S), executed(A2,S), A1!=A2, $\label{ag51}$ 
   action(A1,D1), action(A2,D2), state(S).$\label{ag52}$ 
  \end{lstlisting}
Adding positive postconditions of actions:   $ps(a)^{+} = X \Leftrightarrow \forall x \in X \cdot$
   \begin{lstlisting}[name=main,mathescape]
holdsat($x$,S2) :- executed($a$,S1), action($a,d$), $\label{ag81}$
		 state(S1), state(S2), S2=S1+$d$. $\label{ag82}$
  \end{lstlisting}
Terminating negative post conditions of actions:  $ps(a)^{-} = X \Leftrightarrow \forall x \in X \cdot$
   \begin{lstlisting}[name=main,mathescape]
terminated($x$,S2) :- executed($a$,S1), action($a,d$), $\label{ag91}$
                    state(S1), state(S2), S2=S1+$d$-1. $\label{ag92}$
   \end{lstlisting}
Creating the goals:   $\forall g \in G$
\begin{lstlisting}[name=main,mathescape]
goal($g,v$). $\label{goal0}$
satisfied($g$,S) :- EX($g^{+}$,S), not EX($g^{-}$,S), state(S). $\label{goal1}$
\end{lstlisting}
\end{figure*}

\begin{figure}
Creating the norms: $\forall n = \langle o|f, a_{sub}, a_{con},dl, c \rangle \in N$
 \begin{lstlisting}[name=main,mathescape]
norm($n,c$). $\label{norm1}$
\end{lstlisting}  
 $\forall n = \langle o, a_{sub}, a_{con},dl, c \rangle \in N$
 \begin{lstlisting}[name=main,mathescape]
holdsat(o($n$,S1,$a_{sub},dl$+S2),S2) :- executed($a_{con}$,S1), $\label{norm21}$ 
             action($a_{con},d$), S2=S1+$d$,state(S1), state(S2). $\label{norm22}$
cmp(o($n$,S1,$a,$DL),S2) :- holdsat(o($n$,S1,$a,$DL),S2), $\label{norm41}$
   executed($a$,S2),action($a,d$),state(S1),state(S2),S2!=DL. $\label{norm42}$
terminated(o($n$,S1,$a,$DL),S2) :- cmp(o($n$,S1,$a,$DL),S2),$\label{norm51}$
                              state(S1), state(S2).$\label{norm52}$
vol(o($n$,S1,$a,$DL),S2) :- holdsat(o($n$,S1,$a,$DL),S2), DL=S2, $\label{norm61}$
                       state(S1), state(S2).$\label{norm62}$ 
terminated(o($n$,S1,$a,$DL),S2) :- vol(o($n$,S1,$a,$DL),S2),$\label{norm71}$
                               state(S1), state(S2).$\label{norm72}$ 
\end{lstlisting}                               
$\forall n = \langle f, a_{sub}, a_{con},dl, c \rangle \in N$                              
\begin{lstlisting}[name=main,mathescape] 
holdsat(f($n$,S1,$a_{sub},dl$+S2),S2) :- executed($a_{con}$,S1), $\label{norm81}$
              action($a_{con},d$),S2=S1+$d$,state(S1), state(S2). $\label{norm82}$ 
cmp(f($n$,S1,$a,$DL),S2) :- holdsat(f($n$,S1,$a,$DL),S2),  $\label{norm101}$
                action($a,d$), DL=S2, state(S1), state(S2).  $\label{norm102}$ 
terminated(f($n$,S1,$a,$DL),S2) :- cmp(f($n$,S1,$a,$DL),S2),$\label{norm110}$
                              state(S1), state(S2).$\label{norm111}$
vol(f($n$,S1,$a,$DL),S) :- holdsat(f($n$,S1,$a,$DL),S2),  $\label{norm121}$   
             executed($a$,S2),state(S1) state(S2), S2!=DL.$\label{norm122}$         
terminated(f($n$,S1,$a,$DL),S2) :- vol(f($n$,S1,$a,$DL),S2),  $\label{norm131}$  
                              state(S1), state(S2).$\label{norm132}$ 
\end{lstlisting}
Plans need to satisfy at least one goal
\begin{lstlisting}[name=main,mathescape] 
satisfied($g$) :- satisfied($g$,S), state(S).  $\label{goal2}$
:- not satisfied($g1$), ... , not satisfied(gm). $\label{p1}$
\end{lstlisting}
Avoiding conflicting actions: $\forall \; (a_{1}, a_{2}) \in \mathit{cf}_{action}$
\begin{lstlisting}[name=main,mathescape] 
:- inprog($a1$,S),inprog($a2$,S),action($a1,d1$), action($a2,d2$), $\label{ag71}$
        state(S).$\label{ag72}$
  \end{lstlisting}

\caption{Mapping $P = (\mathit{FL}, I, A, G, N)$ to its Corresponding Computational Model $\Pi_{PBase}$}
\label{mapping}
\end{figure}

\subsubsection{Goals and Norms} \label{gn}
Line \ref{goal0} encodes goal $g$ with value $v$ as a fact. Goal $g$ is satisfied in state $s$ if $s \models g$. This is expressed in 
line \ref{goal1}, where $g^{+}$ and $g^{-}$ are the positive and negative literals in the set $g$.

For the norms we note
that, following Definitions \ref{oc}--\ref{pv}, compliance and violation of a norm can be established based on the start state of action's execution that is the subject of the norm, or at the end state of action's execution. In the encoding we show an implementation of the former; the latter can be catered for in a similar fashion. 

Lines \ref{norm1}--\ref{norm132} deal with obligations and prohibitions of the form $n=\langle d\_o, a_{\mathit{con}}, a_{\mathit{sub}},\mathit{dl}, c\rangle$. Line \ref{norm1} encodes norm $n$ with penalty cost $c$ upon violation. In order to implement the concepts of norm compliance and violation for instantiated norms, we introduce a normative fluent that holds over the compliance period. The compliance period begins from the state in which action $a_{\mathit{con}}$'s execution ends. The compliance period then ends within $\mathit{dl}$ time units of the end of action $a_{\mathit{con}}$, which is denoted as $\mathit{dl}'$ in the normative fluent. 
For instance, fluent 
$o(n_1,s',a_{\mathit{sub}},\mathit{dl}')$ expresses that the instance of norm $n_1$ that was activated in state $s'$, obliges the agent to execute action $a_{\mathit{sub}}$ before deadline $\mathit{dl}'$. The state in which the norm is activated is a part of the fluent to distinguish different activations of the same norm from one another. For example, fluent $o(n_1,s'',a_{\mathit{sub}},\mathit{dl}'')$ refers to a different instance of norm $n1$ that was activated in $s''$. An obligation fluent  denotes that action 
$a_{\mathit{sub}}$'s execution should begin before deadline $\mathit{dl}'$ or be subject to 
violation, while prohibition fluent $f(n_2,s', a_{\mathit{sub}},\mathit{dl}')$ denotes that action 
$a_{\mathit{sub}}$ should not begin before deadline $\mathit{dl}'$ or be subject to violation. 
Lines \ref{norm21}--\ref{norm22} and \ref{norm81}--\ref{norm82} establish respectively the
obligation and prohibition fluents that hold for the duration of the compliance period.

In terms of compliance, if the obliged action \emph{starts} during the compliance period in which the obligation fluent holds, the obligation is complied with (line \ref{norm41}--\ref{norm42}). Compliance is denoted by the atom $\texttt{cmp}$.
The obligation fluent is terminated in the same state that 
compliance is detected (lines \ref{norm51}--\ref{norm52}). If the deadline expires and the obligation fluent still holds, it means that the compliance never occurred during the compliance period and the norm  is therefore violated (lines \ref{norm61}--\ref{norm62}). The atom $\texttt{vol}$ denotes violation. The obligation fluent is terminated when the deadline expires and the norm is violated (lines \ref{norm71}--\ref{norm72}).

On the other hand, a prohibition norm is violated if the forbidden action \emph{begins} during the compliance period in which the prohibition fluent holds (lines \ref{norm121}--\ref{norm122}). As with the obligation norms, after being violated, the prohibition fluent is terminated (lines \ref{norm131}--\ref{norm132}). If the deadline expires and the prohibition fluent still holds, that means the prohibited action did not begin during the compliance period and the norm is therefore complied with (lines \ref{norm101}--\ref{norm102}). The obligation fluent 
is terminated in the same state that compliance is detected (lines \ref{norm110}--\ref{norm111}).

\subsubsection{Mapping Answer Sets to Plans}\label{mp}
Having implemented the components of $P=(\mathit{FL}, \Delta, A, G, N)$, we now encode the criteria for a sequence of actions to be identified as a plan and a solution to $P$. 
The rule in line \ref{p1} is responsible for constraining the answer sets to those that fulfill at least one goal. This is done by excluding answers that do not satisfy any goal. 
The input for this rule is provided in line \ref{goal2}, where goals are marked 
as satisfied if they are satisfied in at least one state.
Prevention of concurrent conflicting actions is achieved via 
lines \ref{ag71}--\ref{ag72} which establish that two such actions cannot be in progress simultaneously.
This concludes the mapping of a formal planning problem to its computational counterpart in \AnsPrologs. For a problem $P$ we refer to the program 
consisting of lines \ref{time1}--\ref{ag72} as $\Pi_{\mathit{PBase}}$.

As mentioned in Section \ref{sec:scenario}, the formulation of our Disaster scenario is provided in \ref{ForScen}. The mapping of the scenario to its computational model follows in 
\ref{LpScen}.



\subsubsection{Soundness and Completeness of Implementation}
The following theorems state the correspondence between the solutions for problem $P$ and answer sets of program $\Pi_{PBase}$.

\begin{theorem}[Soundness]\label{theoremmainSound} 
Let $P = (\mathit{FL}, I, A, G, N)$ be a normative practical reasoning problem with $\Pi_{\mathit{PBase}}$ as its corresponding \AnsProlog program. Let $Ans$ be an answer set of
$\Pi_{PBase}$, then a set of atoms of the form $executed(a_{i},t_{a_{i}})$ $\in Ans$ encodes a solution
to $P$. 
\end{theorem}
The proof of this theorem is presented in \ref{proofSound}. This is a proof by structure that explains how the structure of $\Pi_{PBase}$ satisfies the conditions that identifies a sequence of actions as a plan.

\begin{theorem}[Completeness]\label{theoremmainComp} 
Let  $\pi=\langle (a_{0},0),\cdots,(a_{n},t_{a_{n}})\rangle$ be a  plan for $P = (\mathit{FL}, I, A, G, N )$. Then there exists an answer set of $\Pi_{PBase}$ containing atoms $\mathit{executed}(a_{i},t_{a_{i}})$ that correspond to $\pi$. 
\end{theorem}
The proof of this theorem is presented in \ref{proofComp}. In this proof the program is first transformed to a program without any naf-literals and choice rules. We then take a candidate answer set for the program and show that it is a minimal model for the transformed program. 
 
\subsubsection{Optimal Plans}\label{subsec:opt}

In order to find optimal plans in Figure \ref{optimised} we show how to encode the utility function defined by Eq.~\ref{utility}. The sum of values of goals satisfied in a plan is calculated in line~\ref{opt1}, where we use an ASP built-in aggregate \asp{\#sum}. This aggregate is an operation on a set of weighted literals that evaluates to the sum of the weights of the literals. We first assign the value of goals as the weight of literals \asp{satisfied(G)} and then use \asp{\#sum} to compute the sum of value of all goals satisfied. 

The sum of costs of norms violated in a plan is calculated in line~\ref{opt2} using the same aggregate. However, the weight of the literal is the cost of punishment of the norms. The input for this line is given in lines \ref{norm6-7} and \ref{norm12-13}, where violated norms are distinguished based on the combination of the norm id $n$ and the state $s$ in which they are instantiated. Having calculated \asp{value(TV)} and \asp{cost(TC)}, the utility of a plan is computed in line \ref{op3}, which is subject to a built-in optimisation statement in the final line. This optimisation statement identifies an answer set as optimal if the sum of weights of literals that hold is maximal with respect to all answer sets of the program. By assigning \asp{U} as the weight of literal \asp{utility(U)} we compute the answer set that maximises the utility.

\begin{figure}
\begin{lstlisting}[firstnumber=last,mathescape]
value(TV) :- TV = #sum {V: goal(G,V), satisfied(G)}.$\label{opt1}$  
violated(N,S1) :- vol(o(N,S1,$a,$DL),S2), state(S1;S2). $\label{norm6-7}$  
violated(N,S1) :- vol(f(N,S1,$a,$DL),S2), state(S1;S2). $\label{norm12-13}$  
cost(TC) :- TC = #sum{C,S:violated(N,S),norm(N,C)}.$\label{opt2}$
utility(TV-TC) :- value(TV), cost(TC).$\label{op3}$  
#maximize {U:utility(U)}. $\label{opt4}$  
\end{lstlisting}
\caption{Optimal Plans for $P = (\mathit{FL}, I, A, G, N)$}
\label{optimised}
\end{figure}

Let program $\Pi_{P} = \Pi_{PBase} \cup \Pi^{*}_{P}$, where $\Pi^{*}_{P}$ consists of lines \ref{opt1}--\ref{opt4}.
The following theorem states the correspondence between the plans for problem $P$ and answer sets of program $\Pi_{P}$. 

\begin{theorem}\label{theoremmainSound2} 
Given a normative practical reasoning problem  $P = (\mathit{FL}, I, \allowbreak A, \allowbreak G,\allowbreak N)$, for each answer set $\mathit{Ans}$ of $\Pi$ the set of atoms of the form $\mathit{executed}(a_{i},\allowbreak t_{a_{i}})$ in $Ans$ encodes an optimal solution to $P$. 
Conversely, each solution to the problem $P$ corresponds to a single answer set of $\pi$.
\end{theorem}

This theorem follows immediately from Theorem \ref{theoremmainSound} and \ref{theoremmainComp} and the structure of program $\Pi^{*}_P$.

\section{Related Work} \label{sec:related}
In this section we first review a number of architectures aimed at incorporating normative reasoning into practical reasoning which we group into BDI (Section \ref{BDI}) and non-BDI (Section \ref{NBDI}). Then in Section \ref{NRM} we examine different classifications of these architectures according to the approaches taken. In the same section we also compare these approaches with the approach proposed here. 

\subsection{BDI Architectures with Normative Reasoning}\label{BDI}

There is  a substantial body of work on the integration of norms into the BDI architecture \citep{Rao1995}, but motivation, theory and practice vary substantially.  A key assumption here is that a plan library (i.e. a set of pre-defined plans) exists and the agent uses normative considerations to choose and/or customise a provided plan, rather than generating a norm-compliant plan. We review several normative BDI frameworks and compare them with the approach proposed here.


The BOID architecture \citep{Dastani2001} extends BDI with the concept of obligation and uses agent types such as social, selfish, etc. to handle the conflicts between beliefs, desires, intentions and obligations. For example, selfish agents give priority to their desires in case of conflict, whereas social agents give priority to their obligations.  Since beliefs, desires, intentions and obligations are all represented as a set of rules, priorities are assigned to rules and subsequently used to resolve the conflict. BOID's rich rule-based language makes for a very expressive system, but it is now of largely historical interest, since the implementation is complicated~-- no reference version is currently available~-- and has a high computational complexity.

NoA \citep{Kollingbaum2005} is a normative language and agent architecture. As a language it specifies the normative concepts of obligation, prohibition and permission to regulate a specific type of agent interaction called ``supervised interaction''. As a practical reasoning agent architecture, it describes how agents select a plan from a pre-generated plan library
such that the norms imposed on the agent at each point of time are observed.  NoA agents do not have internal motivations such as goals or values that might conflict with norms, therefore the agent will always be norm compliant.  Publications on NoA do not discuss its evaluation and validate the implementation using examples, which makes the status of the implementation unclear.

$\nu$-BDI \citep{Meneguzzi2015} enables BDI agents to perform normative reasoning for the purpose of customising pre-existing plans that ensure compliance with the set of norms imposed on the agent. That said, there are mechanisms in place to allow norm violation where goal achievement would not otherwise be possible.  In contrast to BOID and NoA, much attention is paid to the practicality and computational efficiency of reasoning in $\nu$-BDI. This is evidenced by the complexity analysis of the algorithms for the norm management mechanism, complemented by empirical analysis, and which both report reasonable (sic) computational costs. 

N-2APL \citep{Alechina2012} is a norm-aware agent programming language based on 2APL \citep{Dastani2008} that supports representation of and reasoning about beliefs, goals, plans, norms, sanctions and deadlines. The N-2APL agents select plans to execute such that they fulfill obligations imposed on agents. The agent can also choose to suppress certain plans to avoid violating prohibitions.  Scheduling of plans is conducted based on plan deadlines or possible sanctions associated with the plans. The scheduling does not concern itself with construction of interleaving plans, thus scheduling boils down to sequencing of plans. Norms in N-2APL are quite simple, being either obligations to carry out and prohibitions not to carry out a specified action. The norms are not conditional (i.e. there is no activation condition defined that triggers the norm) and unlike obligations the prohibitions do not have deadlines.

N-Jason \citep{DBLP:conf/dalt/LeePLDA14} sets out an extension of the Jason \citep{Bordini2007} variant of the BDI architecture to account for norms in plan selection and to handle priorities in plan scheduling.  Like N-2APL, N-Jason enables the underlying implementation of the BDI architecture to carry out norm-aware deliberation and provides a run-time norm execution mechanism, through which new unknown norms~-- as long as they pertain to known actions~-- are recognized and can trigger plans.  To be able to process a norm such as an obligation, which includes a deadline, the agent architecture must be able to reason about deadlines and priorities, and choose between plans triggered by a particular norm. Consequently, N-Jason extends the syntax of the plan library to allow priority annotation, and the scheduling algorithm of AgentSpeak(RT) to operate in the context of Jason/AgentSpeak(L) and 
provide ``real-time agency''.

\subsection{Non-BDI Architectures with Normative Reasoning}\label{NBDI}

A second smaller group of research work considers the problem either from a non-BDI perspective or are agent-architecture agnostic.  One result of not being tied to BDI is not necessarily relying on a pre-generated plan library, which raises issues of how to generate norm-compliant plans.

\cite{Wamberto2011} do utilise a pre-generated plan, like the list above, but take norms into consideration when deciding how to execute the plan with respect to the norms triggered by that plan.  Specifically, the approach aims to adjust the chosen plan to account for the norms that govern the plan actions at each point in time, where the norm expresses constraints on the values that can be assigned to variables in a plan action. The adjustments of values in actions specify how the agent should execute a plan, such that the cost of violated norms is outweighed by the reward from norm compliance. The most preferred plan is therefore the one that maximises this metric.

\cite{Sofia2012} take norms into account in plan generation where the planning problem is specified in PDDL 2.1 \citep{Fox2003}. The normative state of the agent is checked, using the planning tool Metric-FF \citep{Hoffmann2001}, after each individual action; then the planner decides if the agent should comply with a norm or not based on a linear cost function specified in terms of constraints on the states achieved during each plan. Although this mechanism enables an agent to cope with the dynamics of operating in an open environment, checking the normative position of an agent after each action imposes a high computational cost on the plan generation phase.


\cite{Shams2015b} define an approach to practical reasoning that considers norms in both plan generation and plan selection. The agent attempts to satisfy a set of potentially conflicting goals in the presence of norms, as opposed to conventional planning problems that generate plans for a single goal. The main contributions of \cite{Shams2015b} are \begin{inparaenum}[(i)] \item the introduction of an enforcement approach that is a combination of utility-based and pressure-based compliance methods \citep{Lopez2005}, and \item formalising the conflicts between goals, between norms, and between goals and norms. \end{inparaenum}
There is a penalty cost for norm violation regardless of the existence of conflict.  Whenever a norm is triggered, outcomes of norm compliance and violation, and their impacts on the hinderance or facilitation of other goals and norms, are both generated and compared by utility.  In those cases where there are no conflicts and no goals or norms are hindered by the punishment of violation, the loss of utility drives the agent towards compliance. Plans are selected based on the comparison of the utility of the goals satisfied and cost of norms violated in the entire plan.

\subsection{Normative Reasoning Mechanisms}\label{NRM}

We now consider the approaches accounting for norms in practical reasoning in order to uncover similarities across architectures.
Not surprisingly, much work stems from planning and how to take account of norms in the plan identification or construction process.  There are broadly three approaches:

\begin{compactenum}
\item 
Choosing a plan that is norm compliant (e.g., NoA \citep{Kollingbaum2005}), which is a one-off process, that may fail delivering the best (where ``best'' can be defined in various ways) plan available for the situation from those available, and which requires starting again when a plan step fails and the remainder of the plan is invalidated. The main points of difference between NoA and the work presented here are that
\begin{inparaenum}[(i)]
\item NoA agents are BDI specific, 
\item they do not have internal motivations such as goals or values that might conflict with norms, which therefore enables the NoA agent to always comply with norms
\item plans are pre-existing rather than generated.
\end{inparaenum}

\item 
Customising a plan to make it norm compliant (e.g., \citep{Wamberto2011}) is potentially more flexible in making use of available plans (also helping customize existing plans into optimal norm-compliant plans), but otherwise has the same replanning drawback.  In common with \cite{Wamberto2011}, we use the utility of the entire plan in the selection process, but differ in that we generate plans rather than use plans from a library.

\item Generating a plan that is norm compliant (e.g., \citep{Sofia2012,Shams2015b}).  The former addresses on-going compliance and re-planning, putting a high computational overhead on each plan step.  Of necessity, \cite{Sofia2012} can only compute utility on a step-by-step basis, whereas we consider the utility of the whole plan.  \cite{Shams2015b} attempt to balance compliance on the part of the agent (where the agent chooses a norm-compliant action in preference) with enforcement 
(where the agent is discouraged from non-norm-compliance via punishments for norm violation), but is not robust to plan failure. 
Furthermore, in \cite{Shams2015b}, conflict is formulated in advance by taking a static view about conflicts. 
For instance, two goals that are logically inconsistent, cannot be satisfied in the same plan, regardless of the order or choice of actions in a plan. In contrast, in the work presented here, conflicts are not formulated in advance; instead, they are inferred from plans. Therefore, the agent might be able to schedule its actions such that two goals that are logically inconsistent are satisfied in the same plan at different points in time. 
As discussed in Section \ref{semantics}, the norm representation is extended to accommodate compliance and violation in the presence of durative actions more flexibly by allowing compliance to be defined as the start or end of the action that is the subject of the norm.
\end{compactenum}

\begin{table}
\centering
\caption{Summary of Related Frameworks}
\label{sum}
\scriptsize
\begin{tabular}{lccc}
\hline \\
Framework &  Deontic Operator &  Activation Condition & De-activation Condition \\
\hline\\
BOID \citep{Dastani2001} & o & N/A & N/A\\
NoA \citep{Kollingbaum2005} &  o, f, p& state, action & state, action \\
$\nu$-BDI \citep{Meneguzzi2015} &  o, f&  state & state\\
N-2APL \citep{Alechina2012}&  o, f& state  & state\tablefootnote{The de-activation condition only applies to obligations. Prohibitions do not have such a condition.}\\
N-Jason \citep{DBLP:conf/dalt/LeePLDA14} &  o, p, w\tablefootnote{Operator w stands for power and indicates the capability of doing something in prohibitive societies, where actions are not allowed unless empowered and explicitly permitted.}& N/A & temporal constraint\\
\cite{Wamberto2011} & o, f &  N/A & N/A\\
\cite{Sofia2012} &  o, f& state & state \\
\cite{Shams2015b} &  o, f& action & temporal constraint \\
This work & o, f & action & temporal constraint \\
\hline
\end{tabular}
\end{table}

Table \ref{sum} 
shows a summary of related framework to the framework proposed in this paper. For the current work, the parameters compared are the same as \cite{Shams2015b}.
The majority of frameworks, including our framework, deal with obligations and prohibitions. Activation condition, however, varies in those that do support conditional norms. An activation condition presented as an action  can be expressed as a state that satisfies the post-conditions of the action. Unlike many related frameworks, we have exploited the explicit representation of time in formal model to  encode the norm de-activation condition as a time instant. As discussed in Section \ref{norm}, associating a deadline with temporal properties is considered to be realistic and dynamic, in particular when the norms capture the requirements of real-world scenarios \citep{Chesani2013,Kafali2014,Gasparini2015}, such as the disaster scenario we have modelled in this paper. Another important differentiation point between our work and the related ones is that our model is capable of handling durative actions and their execution concurrently, as well as dealing with norm compliance and violation in the presence of durative actions.


\section{Conclusions and Future Work} \label{sec:conclusion}

An agent performing practical reasoning in an environment regulated by norms constantly needs to weigh up the importance of goals to be satisfied and norms with which to comply. This decision process is only possible when the agent has access to the set of all possible plans available and the agent can ascertain the impact of its decision on entire plans. This research  offers means to capture and measure the impact via utility functions, offering numeric metrics, so that the decision problem can be reformulated as choosing a plan from a set of generated plans, which maximises its overall utility. While the literature we have surveyed contains practical reasoning frameworks that take into account normative considerations, they are limited in several ways, and we have contrasted them with our approach in this article. 

The majority of these frameworks are limited to a specific type of agent architecture, mostly BDI, (e.g., \citep{Dastani2001, Kollingbaum2005, Meneguzzi2015}). In our research we do not assume any specific architecture. Instead, we adopt a realistic view that agents have capabilities (encoded as the actions they may perform), and they have internal motivations (encoded as the goals of a planning problem). This leaves the option of extending current work to a multi-agent setting where agents might not necessarily have the same architecture.

In the approaches set out in the literature the attitude agents have towards norms is often one of compliance, meaning that their plans are often selected or, in some approaches, customised, to ensure norm compliance, (e.g., \citep{Kollingbaum2005,Alechina2012,Wamberto2011}). We argue that in certain situations an agent might be better off violating a norm which, if followed, would make it impossible for the agent to achieve an important goal or complying with a more important norm; we enable agents to compare the relative importance of goals and norms via their utilities. Consequently, through our practical reasoning process agents consider {\em all\/} plans (i.e., sequences of actions), including those leading to norm compliance and violation; each plan gets an associated overall utility for its sequences of actions, and norms followed/violated, and agents can decide which of them to pursue. The plan an agent chooses to follow is not necessarily norm-compliant, however, our mechanism guarantees that the decision will maximise the overall plan/norm utility, which justifies the occasional violation of norms as the plan is followed.

We see several interesting avenues for future work. Our research currently addresses normative practical reasoning in a single-agent setting, extending to a multi-agent setting seems a natural next step. This idea can be explored both when the agents are collaborating to fulfill common goals, as well as when they are competing to use resources to fulfill their individual goals. In the former case, the best course of action can be identified as one that maximises the overall utility of the system. In the latter, game-theoretic approaches can be utilised to identify a solution that ensure relative optimality for the individuals (e.g. \citet{Agotnes2007}). Another possibility to explore in a multi-agent setting is to infer conflicts between goals, between norms and between goals and norms by analysing the overall set of possible plans. The inferred conflicts can guide the process of re-engineering of the system toward a more social and norm compliant system (e.g. \citep{Purvis2013}).

We note the relative limitations of our norm representation. Although our approach addressed action-based norms, we envisage how it can be extended and adapted to handle state-based norms. Our Def.~\ref{def:norms} needs to cater for formulae to represent both the norm activation condition, $a_{con}$, and the norm subject, $a_{sub}$, instead of actions. A combination of action- and state-based norms (e.g. \citet{DeVos2013}) enriches the norm representation as well as normative reasoning. Also, the norm representation language can be extended to cater for deadlines that are expressed as reaching a state\footnote{In such cases the deadline is referred to as a norm termination condition.} rather than a time instance. For instance, an obligation to open a dam on a river can come in force when the water level is above a certain point, and subsequently  terminated when the water level drops below a certain level, regardless of how long it takes for that to happen. We would also like to include permission norms in addition to obligations and prohibitions. The modelling of permissions as exceptions to obligations and prohibitions has been used to justify violations under specific circumstances, (e.g. \cite{2Luck2010, Criado2012}). It is also used as a means to handle the uncertainty and  incompleteness of the knowledge of the environment the agents operate in \citep{Alrawagfeh2014}.

Finally, our implementation should be seen as a proof-of-concept that, apart from replanning, provides a provable computational realisation of all aspects of the formal model. In future, we aim at extending the implementation to accommodate replanning when a plan in progress is interrupted for any reason. 
The formal model is implementation language neutral so other implementation languages could be used.

\appendix
\section{Formulation of the Disaster Scenario} \label{ForScen}

We provide a formalisation of the scenario set out in Section \ref{sec:scenario}.
Let $P = (\mathit{FL}, \Delta, A, G, N )$ be the normative practical reasoning problem for the disaster scenario such that:
\begin{itemize}
\item $\mathit{FL}=
\left\{
\begin{array}{c}
\mathit{shockDetected}, \mathit{poisonDetected}, \mathit{waterSupplied}, \\
\mathit{areaSecured}, \mathit{evacuated}, \mathit{shockDetected}, \\
\mathit{shelterBuilt}, \mathit{populated}, \mathit{wounded}, \\
\mathit{earthquakeDetected}, \mathit{medicineSupplied}, \\
\mathit{noAccess}, \mathit{medicsPresent}
\end{array}
\right\}$

\item $\Delta=
\left\{
\begin{array}{c}
\mathit{earthquakeDetected}, \mathit{medicsPresent}, \\
\mathit{wounded}, \mathit{populated}, \mathit{waterSupplied}
\end{array}
\right\}$

\item $A=
\left\{
\begin{array}{c}
\mathit{detectShock}, \mathit{detectPoison}, \mathit{stopWater}, \\
\mathit{buildShelter}, \mathit{evacuate}, \mathit{getMedicine}, \mathit{secure}
\end{array}
\right\}$ where

$\mathit{detectShock} = \langle \{\},\{\mathit{shockDetected}\}, 1\rangle.$

$\mathit{detectPoison} = \langle \{\},\{\mathit{poisonDetected}\}, 1\rangle.$

$\mathit{stopWater} = 
\left\langle 
\left\{
\begin{array}{c}
\mathit{poisonDetected}, \\
\mathit{waterSupplied}
\end{array}
\right\}, \{\neg waterSupplied\},1
\right\rangle$

$\mathit{buildShelter} = 
\left\langle 
\left\{
\begin{array}{c}
\mathit{areaSecured}, \\
\mathit{evacuated},\\
\neg \mathit{shockDetected} 
\end{array}
\right\},
\left\{
\begin{array}{c}
\mathit{shelterBuilt}, \\
\neg \mathit{evacuated}
\end{array}
\right\}, 4
\right\rangle$

$\mathit{evacuate} = 
\left\langle 
\left\{
\begin{array}{c}
\mathit{shockDetected}, \\
\mathit{populated}
\end{array}
\right\},
\left\{
\begin{array}{c}
\mathit{evacuated}, \\
\neg \mathit{populated}
\end{array}
\right\},5
\right\rangle$

$\mathit{getMedicine} = 
\left\langle 
\left\{
\begin{array}{c}
\mathit{earthquakeDetected}, \\
\mathit{wounded}
\end{array}
\right\},\{\mathit{medicine}\},3
\right\rangle$

$\mathit{secure} = 
\left\langle 
\{\mathit{evacuated}\},
\left\{
\begin{array}{c}
\mathit{areaSecured},\\
\mathit{noAccess}
\end{array}
\right\},3
\right\rangle$


\item $G=\{\mathit{runningHospital}, \mathit{organiseSurvivorCamp}\}$, where:

$\mathit{runningHospital} =
\left\langle 
\left\{
\begin{array}{c}
\mathit{medicsPresent}, \\
\mathit{waterSupplied}, \\
\mathit{medicineSupplied}
\end{array}
\right\}, 25 
\right\rangle$

$\mathit{organiseSurvivorCamp} =
\left\langle 
\left\{
\begin{array}{c}
\mathit{areaSecured}, \\
\mathit{shelterBuilt}
\end{array}
\right\}, 18 
\right\rangle$

\item $N=\{n_{1}, n_{2}\}$, where:

$n_{1} = \langle f, \mathit{detectShock}, \mathit{buildShelter}, 3, 5 \rangle$

$n_{2} = \langle o, \mathit{detectPoison}, \mathit{stopWater},2, 10\rangle$

\end{itemize}

\section{Mapping of Our Disaster Scenario}\label{LpScen}
The formal specification of our disaster scenario (Section\ref{sec:scenario}) as provided in the previous section can be mapped to its corresponding \AnsProlog program following the rules given in Figure \ref{mapping} on page \pageref{mapping}. The corresponding program is shown in Figures \ref{lp1}-\ref{lp5}. Optimisation conditions are shown in Figure \ref{opt}. The visualisation of the three answer sets of the program is displayed in figures \ref{vis1}--\ref{vis3}, where arcs show the actions in progress and the boxes below each state, show the fluents that hold in that state. The fluents in bold are the fluents that are added to the state, while the crossed fluents are the terminated ones. Norms violated in a state are highlighted in red and goals satisfied are highlighted in green. Applying the optimisation statements  in Figure \ref{opt}, the utility of each plan presented by each answer set is calculated as below, making the plan presented by answer set 3 the optimal plan.

\begin{description}
\item \textbf{Utility of plan presented by answer set 1, Figure \ref{vis1}:}\\
$Utility(\pi_{Ans1}) = v(runningHospital)=25$
\item \textbf{Utility of plan presented by answer set 2, Figure \ref{vis2}:}\\
$Utility(\pi_{Ans2}) = v(runningHospital) + v(organiseSurvivorCamp) - c(n_{1}) - c(n_{1})=25+18-5-5=33$
\item \textbf{Utility of plan presented by answer set 3, Figure \ref{vis3}:}\\
$Utility(\pi_{Ans3}) = v(runningHospital) + v(organiseSurvivorCamp)=43$
\end{description}

\begin{figure}
\begin{minipage}{0.99\textwidth}
\footnotesize
\begin{center}
\begin{lstlisting}[xleftmargin=2em,linerange={1-35}]
%Rules for states

state(0..13).
holdsat(earthquakeDetected,0).
holdsat(medicsPresent,0).
holdsat(wounded,0).
holdsat(populated,0).
holdsat(waterSupplied,0).

%Inertial fluents

holdsat(X,S2) :- holdsat(X,S1), not terminated(X,S1), 
                 state(S1), state(S2), S2=S1+1.

%Rules for Translating Actions

action(detectShock,1).
pre(detectShock,S) :- state(S).
holdsat(shockDetected,S2) :- 
   executed(detectShock,S1), action(detectShock,1), 
   state(S1), state(S2), S2=S1+1.

action(detectPoison,1).
pre(detectPoison,S) :- state(S).
holdsat(poisonDetected,S2) :- 
        executed(detectPoison,S1),
        action(detectPoison,1), state(S1), 
        state(S2), S2=S1+1.

action(stopWater,1).
pre(stopWater,S) :- holdsat(poisonDetected,S),
            holdsat(waterSupplied,S),state(S).
terminated(waterSupplied,S2) :- 
       executed(stopWater,S1), action(stopWater,1), 
       state(S1), state(S2), S2=S1.
\end{lstlisting}
\end{center}
\end{minipage}
\caption{The Corresponding \AnsProlog of The Disaster Scenario Part 1}\label{lp1}
\end{figure}

\begin{figure}
\begin{minipage}{0.99\textwidth}
\footnotesize
\begin{center}
\begin{lstlisting}[xleftmargin=2em,firstnumber=last,linerange={37-73}]
action(buildShelter,4).
pre(buildShelter,S) :- holdsat(areaSecured,S),
	       holdsat(evacuated,S), state(S).
holdsat(shelterBuilt,S2) :- 
 executed(buildShelter,S1), action(buildShelter,4), 
 state(S1), state(S2), S2=S1+4.
terminated(evacuated,S2) :- 
 executed(buildShelter,S1), action(buildShelter,4), 
 state(S1), state(S2),S2=S1+3.

action(evacuate,5).
pre(evacuate,S) :- holdsat(populated,S),  
                   holdsat(shockDetected,S),state(S).
holdsat(evacuated,S2) :- executed(evacuate,S1),
        action(evacuate,5), state(S1), state(S2),
        S2=S1+5.
terminated(populated,S2) :- executed(evacuate,S1),
        action(evacuate,5), state(S1), state(S2),
        S2=S1+4.

action(getMedicine,3).
pre(getMedicine,S) :- 
                   holdsat(earthquakeDetected,S),
                   holdsat(wounded,S),state(S).
holdsat(medicineSupplied,S2) :- 
		  executed(getMedicine,S1),
                  action(getMedicine,3), 
                  state(S1), state(S2), S2=S1+3.

action(secure,3).
pre(secure,S) :- holdsat(evacuated,S),state(S).
holdsat(areaSecured,S2) :- executed(secure,S1),
        action(secure,3), state(S1), state(S2), 
        S2=S1+3.
holdsat(noAccess,S2) :- executed(secure,S1),
      action(secure,3), state(S1), state(S2), 
      S2=S1+3.
\end{lstlisting}
\end{center}
\end{minipage}
\caption{The Corresponding \AnsProlog of The Disaster Scenario Part 2}\label{lp2}
\end{figure}

\begin{figure}
\begin{minipage}{0.99\textwidth}
\footnotesize
\begin{center}
\begin{lstlisting}[xleftmargin=2em,firstnumber=last,linerange={75-110}]
{executed(A,S)} :- action(A,D), state(S).
inprog(A,S2) :- executed(A,S1), action(A,D), 
            state(S1), state(S2), S1 <= S2, S2< S1+D.
:- inprog(A,S), action(A,D), not pre(A,S), state(S).

% This rule ensures that not more than one action can
% be executed per state.
:- executed(A1,S), executed(A2,S), A1!=A2,   
   action(A1,D1), action(A2,D2), state(S).  

% These rules ensure that there is no action executed
% or is in progress in the final state. 
:- executed(A,13), action(A,D).
:- inprog(A,13), action(A,D).

% This rule ensures that the agent is not 
% idle at any point.
alpha(S) :- inprog(A,S), action(A,D), state(S).
:-  not alpha(S), state(S), S<13.

%Rules for Translating Norms 

norm(n1,5).
norm(n2,10).

holdsat(f(n1,S1,buildShelter,3+S2),S2) :-
     executed(detectShock,S1), 
     action(detectShock,1),     
     S2=S1+1, state(S1), state(S2).
cmp(f(n1,S1,buildShelter,DL),S2) :- 
     holdsat(f(n1,S1,buildShelter,DL),S2),     
     action(buildShelter,4), DL=S2, 
     state(S1), state(S2).
terminated(f(n1,S1,buildShelter,DL),S2) :- 	
     cmp(f(n1,S1,buildShelter,DL),S2), 
     state(S1), state(S2).
\end{lstlisting}
\end{center}
\end{minipage}
\caption{The Corresponding \AnsProlog of The Disaster Scenario Part 3}\label{lp3}
\end{figure}

\begin{figure}
\begin{minipage}{0.99\textwidth}
\footnotesize
\begin{center}
\begin{lstlisting}[xleftmargin=2em,firstnumber=last,linerange={111-146}]
vol(f(n1,S1,buildShelter,DL),S2) :-
     holdsat(f(n1,S1,buildShelter,DL),S2),   
     executed(buildShelter,S2),state(S1), 
     state(S2), S2!=DL.
terminated(f(n1,S1,buildShelter,DL),S2) :- 
     vol(f(n1,S1,buildShelter,DL),S2), 
     state(S1), state(S2).

holdsat(o(n2,S1,stopWater,2+S2),S2) :-
     executed(detectPoison,S1),
     action(detectPoison,1), S2=S1+1,
     state(S1), state(S2).
cmp(o(n2,S1,stopWater,DL),S2) :-
     holdsat(o(n2,S1,stopWater,DL),S2),
     executed(stopWater,S2),
     action(stopWater,1),state(S1), 
     state(S2), S2!=DL.
terminated(o(n2,S1,stopWater,DL),S2) :-
     cmp(o(n2,S1,stopWater,DL),S2), state(S1),
     state(S2).
vol(o(n2,S1,stopWater,DL),S2) :-
     holdsat(o(n2,S1,stopWater,DL),S2), DL=S2,
     state(S1), state(S2).
terminated(o(n2,S1,stopWater,DL),S2) :-
     vol(o(n2,S1,stopWater,DL),S2), state(S1),
     state(S2).

% Generating Solution for problem P

goal(runningHospital,25).
satisfied(runningHospital,S) :- 
 holdsat(medicsPresent,S),holdsat(waterSupplied,S), 
 holdsat(medicineSupplied,S), state(S).
satisfied(runningHospital) :- 
             satisfied(runningHospital,S), state(S).
\end{lstlisting}
\end{center}
\end{minipage}
\caption{The Corresponding \AnsProlog of The Disaster Scenario Part 4}\label{lp4}
\end{figure}

\begin{figure}
\begin{minipage}{0.99\textwidth}
\footnotesize
\begin{center}
\begin{lstlisting}[xleftmargin=2em,firstnumber=last,linerange={147-166}]
goal(organiseSurvivorCamp,18).
satisfied(organiseSurvivorCamp,S) :- 
                  holdsat(areaSecured,S),
                  holdsat(shelterBuilt,S), state(S).
satisfied(organiseSurvivorCamp) :-
     satisfied(organiseSurvivorCamp,S), state(S).

:- not satisfied(runningHospital), 
   not satisfied(organiseSurvivorCamp).

:- inprog(detectShock,S), inprog(buildShelter,S), 
   action(detectShock,1), action(buildShelter,4),
   state(S).

:- inprog(buildShelter,S),  inprog(evacuate,S),
   action(buildShelter,4), action(evacuate,5), 
   state(S).

:- inprog(buildShelter,S),  inprog(secure,S),
  action(buildShelter,4), action(secure,3), state(S).
\end{lstlisting}
\end{center}
\end{minipage}
\caption{The Corresponding \AnsProlog of The Disaster Scenario Part 5}\label{lp5}
\end{figure}

\begin{figure}
\begin{minipage}{0.99\textwidth}
\footnotesize
\begin{center}
\begin{lstlisting}[xleftmargin=2em,linerange={168-180}]
violated(n1,S1) :- vol(f(n1,S1,buildShelter,DL),S2),
                   state(S1), state(S2).
violated(n2,S1) :- vol(o(n2,S1,stopWater,DL),S2),
                   state(S1), state(S2).

%Optimisation rules
 
value(TV) :- TV = #sum{V: goal(G,V),satisfied(G)}.
cost(TC) :- TC = #sum{C,S:violated(N,S),norm(N,C)}.

utility(TV-TC) :- value(TV),cost(TC).

#maximize{U:utility(U)}.
\end{lstlisting}
\end{center}
\end{minipage}
\caption{Optimisation Rules for The Disaster Scenario}\label{opt}
\end{figure}

\begin{sidewaysfigure}
\includegraphics[scale=1.5]{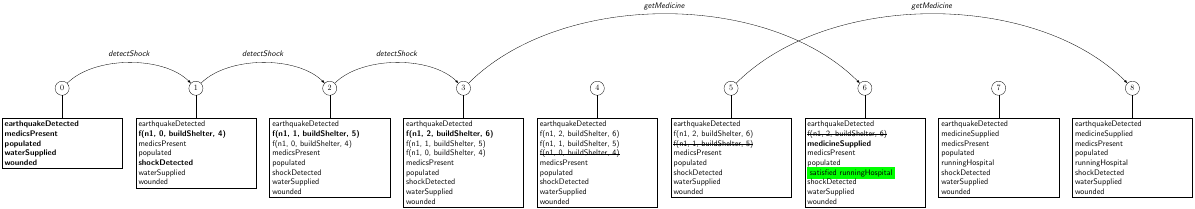}
\vspace{4 mm}\\
\includegraphics[scale=1.5]{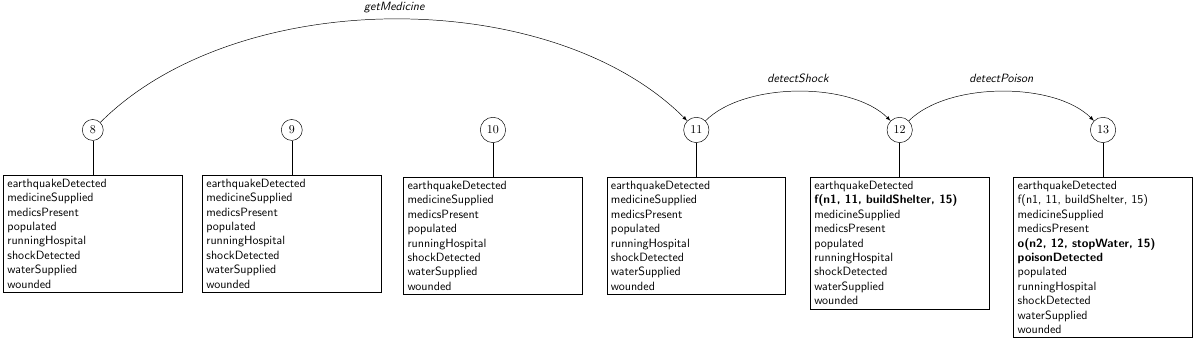}
\caption{Visualisation of answer set 1}
\label{vis1}
\end{sidewaysfigure}

\begin{sidewaysfigure}
\includegraphics[scale=1.5]{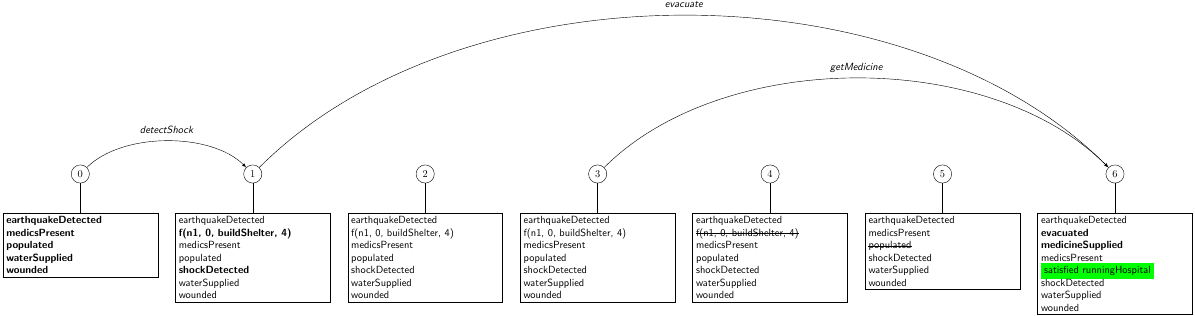}
\vspace{4 mm}\\
\includegraphics[scale=1.5]{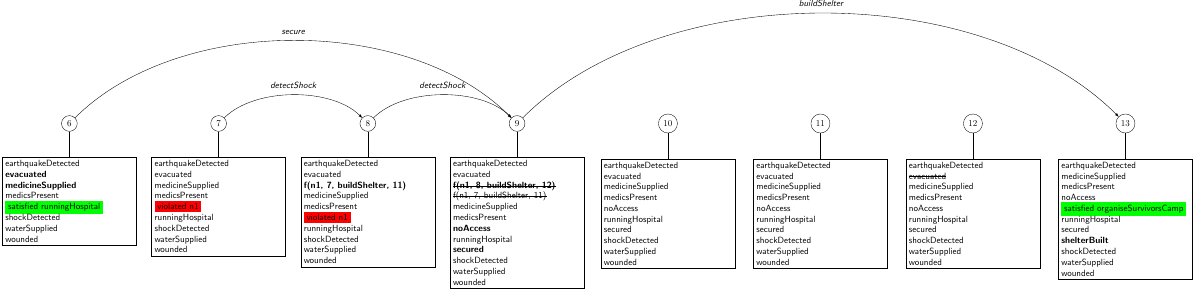}
\caption{Visualisation of answer set 2}
\label{vis2}
\end{sidewaysfigure}

\begin{sidewaysfigure}
\includegraphics[scale=1.5]{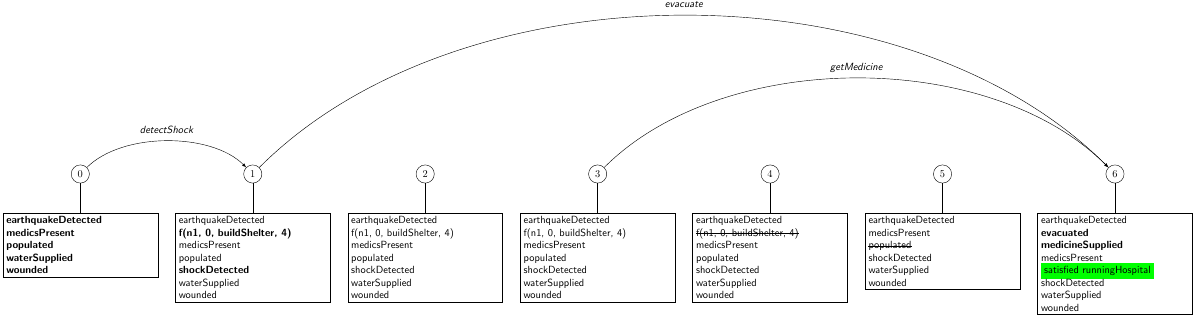}
\vspace{4 mm}\\
\includegraphics[scale=1.5]{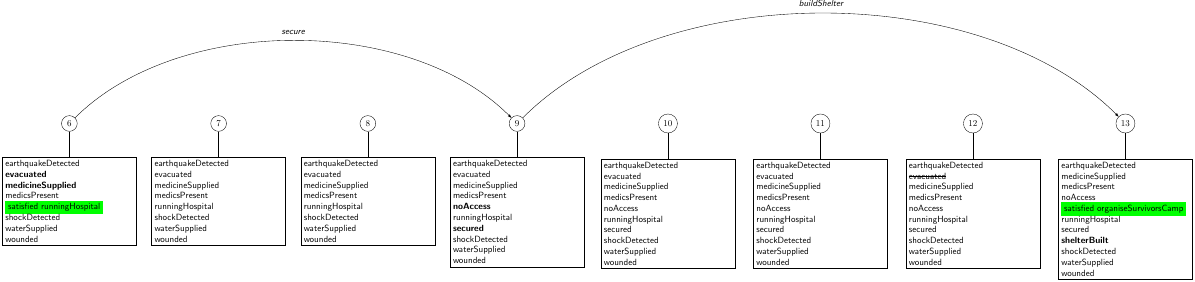}
\caption{Visualisation of answer set 3}
\label{vis3}
\end{sidewaysfigure}

\newpage
\section{Proof of Theorem \ref{theoremmainSound}} \label{proofSound}

\begin{theorem*}[Soundness]
Let $P = (\mathit{FL}, I, A, G, N)$ be a normative practical reasoning problem with $\Pi_{PBase}$ as its corresponding \AnsProlog program. Let $Ans$ be an answer set of
$\Pi_{PBase}$, then a set of atoms of the form $executed(a_{i},$ $t_{a_{i}}) \in Ans$ encodes a solution
to $P$. 
\end{theorem*}

\begin{proof}
We need to prove that program $\Pi_{PBase}$ (Figure \ref{mapping}) generates all 
sequences of actions that meet the criteria that identifies a sequence of actions as a plan, as defined in Definition \ref{plandef}. This implies that the sequence of actions that is a part of the answer set 
satisfies all the criteria to be a solution to the encoded planning program. 

Actions and more precisely the postconditions of actions are what cause the change 
from one state to another one. Line \ref{ag3} generates all sequences of actions. Lines \ref{ag81}--\ref{ag82} changes a state in which some actions end by adding the add postconditions of those actions to the state. In contrast, Lines \ref{ag91} and \ref{ag92} terminate the delete postconditions of actions ending in the next state such that those postconditions do not hold in the following state. If there is no action ending in a state, the state remains unchanged 
as all the fluents are inertial and they hold in the next state unless they are terminated (Line \ref{ag102}). A sequence of actions $\pi =\langle (a_{0},0),$ $\ldots, (a_{n},t_{a_{n}}) \rangle \mbox{ s.t. } \nexists$ $(a_{i},t_{a_{i}}), (a_{j},t_{a_{j}}) \;\widehat{\in}\; \pi \mbox{ s.t. } t_{a_{i}}  \leq t_{a_{j}} < t_{a_{i}}+d_{a_{i}}, (a_{i}, a_{j}) \in \mathit{cf}_{action}$ is a plan for the normative planning problem $P = ( F\!L, \Delta, A, G, N )$ 
if the following conditions hold: 
\begin{compactitem}
 \item fluents in $\Delta$ (and only those fluents) hold in the initial state: $s_{0} = \Delta $.\\
  Line \ref{initial} ensures that all fluents in $\Delta$
 are added to the initial state $s_{0}$. 
 \item the preconditions of action $a_{i}$ hold at time $t_{a_{i}}$ 
 and throughout the execution of $a_{i}$:
 $\forall k \in [t_{a_{i}}, t_{a_{i}}+d_{a_{i}}), s_{k} \models pr(a_{i})$.\\
 Lines \ref{ag4} guarantees that the preconditions of an action hold all 
through its execution. 
 \item plan $\pi$ satisfies a non-empty subset of goals:
$G_{\pi} \neq \emptyset$.\\
Line \ref{p1} indicates that a non-empty subset of goals has to be satisfied in a 
 plan.
 \end{compactitem}
\end{proof}

\section{Proof of Theorem \ref{theoremmainComp}} \label{proofComp}
\begin{theorem*}[Completeness]
Let  $\pi=\langle (a_{0},0),\cdots,(a_{n},t_{a_{n}})\rangle$ be a  plan for 
$P = ( F\!L, I, A, G, N )$. Then, there exists an 
answer set of $\Pi_{PBase}$ containing atoms $executed(a_{i},t_{a_{i}})$ that corresponds to $\pi$. 
\end{theorem*}

For the sake of making this proof self-contained, we first provide a formal definition of an answer set and explain the concept of reduct. As mentioned previously, a number of syntactic language representations for ASP exist. We use AnsProlog which is one of the most common classes of these languages. 
and it has the following elements \citep{Baral2003}:
\begin{compactitem}
  \item[Term:] A term is a constant or a variable or a n-ary function 
  $f(t_{1}, \cdots, t_{n})$, where $f$ is the function symbol and 
  $t_{1}, \cdots, t_{n}$ are terms. Constants start 
  with a lower-case letter, whereas variables start with an upper-case level. 
  A term is \emph{ground\/} if no variable occurs in it.
  \item[Atom:] Atoms are the basic components of the language that can be assigned a 
  truth value as true or false. An atom is a statement of form 
  $A(t_{1}, \cdots, t_{n})$, where $A$ is a predicate symbol and $t_{1}, \cdots, t_{n}$
  are terms.
  \item[Literal:] Literals are atoms or negated atoms.  Atoms are negated using 
\textit{negation as failure} ($not$). $not \; a$ is true if there is no evidence 
proving the truth of $a$. An atom preceded by $not$ is referred to as a naf-literal. 
\end{compactitem} 

The Herbrand universe of language $\mathcal L$ denoted as $HU_{\mathcal L}$ is the 
set of all ground terms which can be formed with the functions and constants in 
 $\mathcal L$. The set of all ground atoms which can be formed with the functions, 
 constants and predicates in $\mathcal L$ is called Herbrand base of language 
 $\mathcal L$ and is denoted using $HB_{\mathcal L}$. 

In Section \ref{(ASP)}, we explained that an AnsProlog program (e.g. $\Pi$) consists of a finite set of rules. 
In order to interpret a rule that contains variables, the rule has to be 
\emph{grounded}. 
The grounding of each rule $r$ in $\Pi$ is then the set of all 
rules obtained from substitutions of elements of $HU_{\Pi}$ for the 
variables in the rule $r$. By grounding all $r \in \Pi$, we obtain $ground(\Pi)$.


The semantics of AnsProlog is defined in terms of \emph{answer sets}.
The answer sets of program $\Pi$, are defined in terms of the answer sets of the
ground program $ground(\Pi)$. 
An AnsProlog program without any naf-literal is denoted as Ansprolog$^{-not}$.
An answer set of an AnsProlog$^{-not}$ program $\Pi$ is 
a minimal subset (with respect to subset ordering) $S$ of $HB$ that is closed 
under $ground(\Pi)$. 
The approach to define the answer sets of an AnsProlog program $\Pi$ is to take a 
candidate answer set $S$ of the program and transform $\Pi$  with respect to $S$
to obtain an  Ansprolog$^{-not}$ denoted by $\Pi^{S}$. $S$ is an answer set of 
$\Pi$ if $S$ is the answer set of AnsProlog$^{-not}$ program $\Pi^{S}$.
This transformation is referred to as Gelfond-Lifschitz \citep{Gelfond1988} 
transformation.

Given an AnsProlog program $\Pi$ and a set $S$ of atoms from $HB_{\Pi}$, 
the Gelfond-Lifschitz \citep{Gelfond1988} transformation $\Pi^{S}$ is obtained by 
deleting:
\begin{compactenum}
  \item each rule that has a $not \; L$ in its body with $L \in S$, and
  \item literals of form $not \; L$ in the bodies of the remaining rules.
\end{compactenum}
The transformation (reduct) of choice rules was not a part of original 
Gelfond-Lifschitz transformation and was introduced later in \citep{Lee2008}.
Recently a simplified  reduct for programs including choice rules is 
proposed by \citet{Law2015} as follows. 
Given an AnsProlog program $\Pi$ - with choice rules- and a set $S$ of atoms from 
$HB_{\Pi}$, the transformation $\Pi^{S}$ is constructed in the following 4 
steps:
\begin{compactenum}
  \item Delete each rule that has a $not \; L$ in its body with $L \in S$.
  \item Delete literals of form $not \; L$ in the bodies of the remaining rules.
  \item for any choice rule $r$, $l\{h_{0}, \cdots, h_{k} \}u :- \; body(r)$,
  such that $l \leq |S \cap \{h_{0}, \cdots, h_{k} \}| \leq u$, replace $r$ with the
  set of rules 
  $\{ h_{i} :- \; body^{+}(r) | h_{i} \in S \cap \{h_{0}, \cdots, h_{k} \} \}$. 
  \item for any remaining choice rules $r$, $l\{h_{0}, \cdots, h_{k} \}u :- 
  \; body(r)$, replace $r$ with the constraint $:- \; body^+(r)$.
\end{compactenum}
After these transformation, the AnsProlog program $\Pi$ is a program without any 
naf-literals and choice rules and it is therefore, an AnsProlog$^{-not}$, for 
which the answer sets are already defined.

\begin{proof}
Let the execution of sequence of actions in  $\pi=\langle (a_{0},0),\cdots,(a_{n},t_{a_{n}})\rangle$
bring about the sequence of states $\langle s_{0}, \cdots s_{q}\rangle$. Let $M_{t}$ be the set of following atoms (and nothing else): 

\begin{fleqn} 
\begin{align} \label{pr0}
\forall k, 0 \leq k \leq q \cdot M_{t} \models state(k)
\end{align}
\end{fleqn}

\begin{fleqn} 
\begin{align} \label{pr3}
\forall k, 0 \leq k \leq q \cdot x \in s_{k} \Rightarrow M_{t} \models holdsat(x,k)
\end{align}
\end{fleqn}

\begin{fleqn} 
\begin{align}\label{pr8}
\forall k, 0 \leq k < q \cdot x \in (s_{k} \setminus s_{k+1}) \Rightarrow M_{t} \models terminated(x,k)  
\end{align}
\end{fleqn}

\begin{fleqn} 
  \begin{align} \label{pr4}
\forall a \in A \cdot M_{t} \models action(a,d)  
 \end{align}
\end{fleqn}

\begin{fleqn} 
\begin{align}\label{pr5}
\forall k, 0 \leq k \leq q \cdot a \in A, pr(a)^{+} \subseteq s_{k}, pr(a)^{-} 
\not\in s_{k} \Rightarrow M_{t} \models pre(a,k)
\end{align}
\end{fleqn}

\begin{fleqn} 
\begin{align}\label{pr6}
\forall k, 0 \leq k < q \cdot (a,k) \;\widehat{\in}\; \pi \Rightarrow M_{t} \models executed(a,k)
\end{align}
\end{fleqn}

\begin{fleqn} 
\begin{align}\label{pr7}
\forall  a.M_{t} \models executed(a,k), \; \forall k, t_{a} \leq k < t_{a}+d(a) \cdot M_{t} \models inprog(a,k)  
\end{align}
\end{fleqn}

\begin{fleqn} 
\begin{align}\label{pr81}
\forall  a.M_{t} \models executed(a,k), \; \forall x \in  
ps(a)^+ \cdot M_{t} \models holdsat(x,k+d(a))
\end{align}
\end{fleqn}

\begin{fleqn} 
\begin{align}\label{pr82}
\forall  a.M_{t} \models executed(a,k), \; \forall x \in  
ps(a)^- \cdot M_{t} \models terminated(x,k+d(a)-1)
\end{align}
\end{fleqn}

\begin{fleqn} 
  \begin{align} \label{}
\forall g \in G \cdot M_{t} \models goal(g,v)  
 \end{align}
\end{fleqn}

\begin{fleqn}
\begin{align} \label{pr9}
\forall k, 0 \leq k \leq q \cdot g \in G, g^{+} \subseteq s_{k}, g^{-} 
\not\in s_{k} \Rightarrow M_{t} \models satisfied(g,k)
\end{align}
\end{fleqn}

\begin{fleqn} 
  \begin{align} \label{pr4}
\forall n \in N \cdot M_{t} \models norm(n,c) 
 \end{align}
\end{fleqn}

\begin{fleqn}
\begin{multline} \label{pr17}
\forall n=\langle o, a_{con}, a_{sub}, dl\rangle \in N. M_{t} \models 
executed(a_{con},t_{a_{con}}), \\ \forall k, t_{a_{con}}+d(a_{con}) \leq k \leq 
t_{a_{con}}+d(a_{con})+dl \cdot \\ M_{t} \models holdsat(o(n,t_{a_{con}},a_{sub},t_{a_{con}}+d(a_{con})+dl),k)
\end{multline}
\end{fleqn}

\begin{fleqn}
\begin{multline}  \label{pr12}
\exists k_{2}, 0 \leq k < q \cdot M_{t} \models holdsat(o(n,k_{1},a,dl'),k_{2}), M_{t} \models executed(a,k_{2}),\\
k_{2} != dl' \Rightarrow  M_{t} \models cmp(o(n,k_{1},a,dl'),k_{2})   
\end{multline}
\end{fleqn}

\begin{fleqn} 
\begin{multline}  \label{pr14}
\exists k_{2}, 0 \leq k < q, M_{t} \models cmp(o(n,k_{1},a,dl'),k_{2}) \Rightarrow \\
M_{t} \models terminated(o(n,k_{1},a,dl'),k_{2})
\end{multline}
\end{fleqn}

\begin{fleqn} 
\begin{align} \label{pr15}
\exists k_{2}, k_{2}=dl', M_{t} \models holdsat(o(n,k_{1},a,dl'),k_{2})
 \Rightarrow M_{t} \models vol(o(n,k_{1},a,dl'),k_{2})  
\end{align}
\end{fleqn}

\begin{fleqn} 
\begin{multline}  \label{pr16}
\exists k_{2}, 0 \leq k_{2} \leq q,   M_{t} \models  vol(o(n,k_{1},a,dl'),k_{2}) \Rightarrow \\
M_{t} \models terminated(o(n,k_{1},a,dl'),k_{2})
\end{multline}
\end{fleqn}

\begin{fleqn} 
\begin{multline} \label{pr11}
\forall n=\langle f, a_{con}, a_{sub}, dl'\rangle \in N. M_{t} \models 
executed(a_{con},t_{a_{con}}), \\ \forall k, t_{a_{con}}+d(a_{con}) \leq k \leq 
t_{a_{con}}+d(a_{con})+dl' \cdot \\ M_{t} \models holdsat(f(n,t_{a_{con}},a_{sub},t_{a_{con}}+d(a_{con})+dl'),k)
\end{multline}
\end{fleqn}

\begin{fleqn}
\begin{multline}\label{pr18}
\exists k_{2}, k_{2}=dl', M_{t} \models holdsat(f(n,k_{1},a,dl'),k_{2}) \Rightarrow \\
M_{t} \models cmp(f(n,k_{1},a,dl'),k_{2})   
\end{multline}
\end{fleqn}

\begin{fleqn}
\begin{multline} \label{pr20}
\exists k_{2}, 0 \leq k_{2} \leq q, M_{t} \models cmp(f(n,k_{1},a,dl'),k_{2}) \Rightarrow \\
M_{t} \models terminated(f(n,k_{1},a,dl'),k_{2})
\end{multline}
\end{fleqn}

\begin{fleqn}
\begin{multline} \label{pr21}
\forall k_{2}, 0 \leq k_{2} < q \cdot  M_{t} \models  holdsat(f(n,k_{1},a,dl'),k_{2}), M_{t} \models executed(a,k_{2}) \Rightarrow \\ M_{t} \models vol(f(n,k_{1},a,dl'),k_{2})  
\end{multline}
\end{fleqn}

\begin{fleqn}
\begin{multline} \label{pr22}
\exists k_{2}, 0 \leq k_{2} \leq q,   M_{t} \models  vol(f(n,k_{1},a,dl'),k_{2})   \Rightarrow \\
M_{t} \models terminated(f(n,k_{1},a,dl'),k_{2})
\end{multline}
\end{fleqn}

\begin{fleqn} 
\begin{align} \label{pr10}
\exists k, 0 \leq k \leq q,  M_{t} \models satisfied(g,k)  \Rightarrow M_{t} \models 
satisfied(g)
\end{align}
\end{fleqn}

We need to prove that $M_{t}$ is an answer set of $\Pi_{PBase}$. Therefore, 
we need to demonstrate that $M_{t}$ is a minimal model for $\Pi_{PBase}^{M_{t}}$.
Let $r \in \Pi_{PBase}^{M_{t}}$ 
be an applicable rule. In order for $M_{t}$ to be a model of $\Pi_{PBase}^{M_{t}}$, we need 
to show that $r$ is 
applied (i.e. $M_{t} \models Head(r)$). We will go through each rule in the same order 
in Lines \ref{time1}--\ref{ag72}. 
\begin{compactitem}
 \item r is of type rule in Line \ref{time1}: fact and automatically applied. 
 \item r is of type rule in Line \ref{initial}: fact and automatically applied. 
 \item r is of type rule in Lines \ref{ag101}--\ref{ag102}: because of  
 Gelfond-Lifschitz transformation, we know that $not \; terminated(x,s)$ is removed from this rule. Combination of \ref{pr3} and 
  \ref{pr8} for $x$ at $k$ gives $M_{t} \models holdsat(x,k+1)$.  
 \item r is of type rule in Line \ref{ag1}:  fact and automatically applied. 
  \item r is of type rule in Line \ref{ag2}: after Gelfond-Lifschitz 
  transformation, from the body and description of this rule we have 
  $a \in A$ and $ pr(a)^{+} \in s_{k}$ and  $pr(a)^{-} \not\in s_{k}$, which 
  with \ref{pr5} implies that $M_{t} \models pre(a,k)$.
 \item r is of type rule in Line \ref{ag3}: any action $a \in A$ can be executed 
 in a state. After the transformation for choice rules, we obtain 
 $\forall (a,k) \;\widehat{\in}\; \pi$ we have $executed(a,k):- \; action(a,d),$ $state(k).$
 and $\forall (a,k) \mbox{ s.t. } (a,k) \;\widehat{\not\in}\; \pi $ we have $:- \; action(a,d), state(k).$
 With \ref{pr6}, we know that $M_{t} \models executed(a,k)$.
  \item r is of type rule in Lines \ref{ag41}--\ref{ag42}: $inprog$ atoms originate from 
  execution of actions. From \ref{pr6} we know that 
  $\forall (a,k) \;\widehat{\in}\;  \pi, M_{t} \models executed(a,k)$. Since $a$ is executed 
  with \ref{pr7} we have 
  $\forall k, t_{a} \leq k < t_{a}+ d(a), M_{t} \models inprog(a,k)$.
  \item r is of type rule in Line \ref{ag4}: the head of this rule  is empty. 
  Since $\pi$ is a plan and we have assumed the preconditions of actions in a 
  plan are hold while the actions are in progress, this rule is applied.
  \item r is of type rule in Lines \ref{ag51}--\ref{ag52}: the head of this rule is empty. 
  Because $\pi$ is a plan and we have assumed that two actions in a plan cannot have 
  exactly the same start state, this rule is applied.
  \item r is of type rule in Lines \ref{ag81}--\ref{ag82}: the body of this rule implies that the 
  add postconditions of an executed action $a$ hold in the state in which 
  the action ends. Since $\forall (a,k) \;\widehat{\in}\; \pi, M_{t} 
  \models executed(a,k)$, with \ref{pr81} we have $\forall x \in  ps(a)^+ \cdot M_{t} 
  \models holdsat(x,k+d(a))$.
  \item r is of type rule in Lines \ref{ag91}--\ref{ag92}:  the body of this rule implies that the 
 delete postconditions of an executed action $a$ are terminated in the state 
 before the end state of the action. Since $\forall (a,k) \;\widehat{\in}\; \pi, M_{t} 
  \models executed(a,k)$, with \ref{pr82} we have $\forall x \in  ps(a)^- \cdot M_{t} 
  \models terminated(x,k+d(a)-1)$.
  \item $r$ is of type rule in Line \ref{goal0}: fact and automatically applied.
  \item r is of type rule in Line \ref{goal1}: after Gelfond-Lifschitz 
  transformation, from the body and description of this rule we have 
  $ g^{+} \in s_{k}$ and $g^{-} \not\in s_{k}$, which with \ref{pr9}
  implies that $M_{t} \models satisfied(g,k)$.
  \item $r$ is of type rule in Line \ref{norm1}: fact and automatically applied.
  \item r is of type rule in Lines \ref{norm21}--\ref{norm22}: the body of the rule implies 
  that normative fluents for obligations hold over the compliance period if 
  the action that is the condition of the norm is executed. If $a_{con}$ for an 
  obligation norm belongs to $\pi$, then based on \ref{pr6} we know that $M_{t} \models 
  executed(a_{con},t_{a_{con}})$. From \ref{pr6} and \ref{pr17} we know that 
  $M_{t} \models holdsat(o(n,k_{1},a_{sub},t_{a_{con}}+d(a_{con})+dl'),k_{2})$ over the period $t_{a_{con}}+d(a_{con}) \leq k_{2} \leq t_{a_{con}}+d(a_{con})+dl'$.  
  \item r is of type rule in Lines \ref{norm41}--\ref{norm42}: this rule states that if the 
  obliged action is executed while the normative fluent holds, the norm is 
  complied with. If $a_{sub}$ is executed in $\pi$ 
  ($M_{t} \models executed(a_{con},t_{a_{con}})$) while the normative fluent 
  in \ref{pr17} holds, with \ref{pr12} we know that $M_{t}$ models the compliance 
  atom.
 \item r is of type rule in Lines \ref{norm51}--\ref{norm52}: complied obligations are  terminated in the 
  compliance state. With \ref{pr12} we know that $M_{t}$ models compliance atoms, 
  and \ref{pr14} implies that they are terminated in the same state.
  \item r is of type rule in Lines \ref{norm61}--\ref{norm62}: if the obligation fluent still holds when the 
  deadline occurs, the obligation is violated. \ref{pr15} implies this is modelled by
  $M_{t}$.
  \item r is of type rule in Line \ref{norm71}--\ref{norm72}: violated obligations are terminated in the 
  violation state. With \ref{pr15} we know that $M_{t}$ models violation atoms, 
  and \ref{pr16} implies that they are terminated.
\item  r is of type rule in Lines \ref{norm81}--\ref{norm82}: 
the body of the rule implies 
  that normative fluents for prohibition norms are hold over the compliance period if 
  the action that is the condition of the norm is executed. If $a_{con}$ for a
  prohibition belongs to $\pi$, then based on \ref{pr6} we know that $M_{t} \models 
  executed(a_{con},t_{a_{con}})$. 
  From \ref{pr6} and \ref{pr11} we know that 
  $M_{t} \models holdsat(f(n,k_{1},a_{sub},t_{a_{con}}+d(a_{con})+dl'),k_{2})$ over the period
   $t_{a_{con}}+d(a_{con}) \leq k_{2} \leq t_{a_{con}}+d(a_{con})+dl'$.  
\item r is of type rule in Lines \ref{norm101}--\ref{norm102}: this rule states that if the 
normative fluent still holds at the end of compliance period, the prohibition 
  is complied with. \ref{pr18} implies that $M_{t}$ models the head of this 
  rule.
\item r is of type rule in Line \ref{norm110}--\ref{norm111}: complied prohibitions are  terminated in the 
  compliance state. With \ref{pr18} we know that $M_{t}$ models compliance atoms, 
  and \ref{pr20} implies that they are terminated and this rule is applied.
\item r is of type rule in Lines \ref{norm121}--\ref{norm122}: this rule states that if the 
  prohibited action is executed while the normative fluent holds, the norm is 
  violated. If $a_{sub}$ is executed in $\pi$ 
  ($M_{t} \models executed(a_{sub},t_{a_{sub}})$) while the normative fluent 
  in \ref{pr11} holds, with \ref{pr21} we know that $M_{t}$ models the violation 
  atom.
\item r is of type rule in Line \ref{norm131}--\ref{norm132}: violated prohibitions are terminated in 
violation state. With \ref{pr21} we know that $M_{t}$ models violation atoms, 
  and \ref{pr22} implies that they are terminated.
\item r is of type rule in Line \ref{goal2}: this rule is applicable whenever a goal is 
satisfied in a state. With \ref{pr9} and \ref{pr10} we can obtain this ($M_{t} \models 
satisfied(g)$).
 \item r is of type rule in Line \ref{p1}: the head of this rule is empty. 
  Because  $\pi$ is a plan it has to satisfy at least one goal, so this rule is applied.
 
\item r is of type rule in Line \ref{ag71}--\ref{ag72}: the head of this rule is empty. 
  Because  $\pi$ is a plan and a plan cannot contain concurrent execution 
  of actions, this rule is applied.
\end{compactitem}
By showing that all the rules, apart from those not applicable, are applied, we have shown that $M_{t}$ is a model for $\Pi_{PBase}^{M_{t}}$.


Now, we need to show that $M_{t}$ is minimal, which means that there exists 
no other model of $\Pi_{PBase}^{M_{t}}$ that is a subset of $M_{t}$.

Let $M \subset M_{t}$ be a model for $\Pi_{PBase}^{M_{t}}$, then there must exist an atom 
$s \in (M_{t} \setminus M)$. If $s$ is an atom that is generated because it is a 
fact, it must belong to $M$ too and if that is not the case, then $M$ cannot be 
a model. We now proceed with the rest of atoms that do not appear as facts:

\begin{compactitem}
   \item $s = executed(a,k)$: $M_{t} \models s$ implies that 
   $r: executed(a,k) :- \; action(a,$ $d),state(k).$, $r \in \Pi_{PBase}^{M_{t}}$.
   If $M \not\models  executed(a,k)$, then $r$ was applicable but not applied, 
   therefore, $M$ is not a model.
   
  \item $s = holdsat(x,k)$: $M_{t} \models s$ implies that one of the following 
  four condition must have occurred:
  \begin{compactitem}
  
  \item \ref{pr3}: if this is the case then  $s$ was true in $s_{k-1}$ (not 
  terminated is removed from the body of this rule because of the Gelfond-Lifschitz). 
  In this case the 
  construction of  $M_{t}$ guarantees $M_{t} \models holdsat(x,k-1)$.
  If $M \not\models holdsat(a,k)$ then rule in Lines \ref{ag101}--\ref{ag102}  
  are applicable but not applied, so $M$ cannot be a model.
  
  \item \ref{pr81}: if this is the case then we know $M_{t} \models executed(a,k)$.
  Earlier we showed that if $M_{t} \models executed(a,k)$, then 
  $M \models executed(a,k)$ too. Thus, if $M \not\models holdsat(x,k+d(a))$ for 
  all $x \in ps(a)^{+}$, then rule in Lines \ref{ag81}--\ref{ag82}
  is applicable but not applied, so $M$ cannot be a model.

\item \ref{pr17}: if this is the case then  $x = o(n,k,a_{sub},
t_{a_{con}}+d(a_{con})+dl)$. Because 
$M_{t}$ is a model then $M_{t} \models executed(a_{con},t_{a_{con}})$. 
So $M \models executed(a_{con},t_{a_{con}})$ too. Therefore, rule in 
Lines \ref{norm21}--\ref{norm22} is applicable and if $M \not\models s$ then 
this rule is not applied and $M$ cannot be a model.

\item \ref{pr11}: if this is the case then $x = f(n,k,a_{sub},t_{a_{con}}+d(a_{con})+dl)$. 
Similar to reasoning above but instead of rule in 
Lines \ref{norm21}--\ref{norm22}, rule in Lines \ref{norm81}--\ref{norm82} is 
applicable. 
\end{compactitem}
 
  \item $s = pre(a,k)$: $M_{t} \models s$ implies that $\forall x \in pr(a)^{+}, 
  x \in s_{k}$ and $\forall x \in pr(a)^{-}, x \not\in s_{k}$. If $M$ is a model 
  then according to the transformed version of rule \ref{ag2}, $M \models pre(a,k)$. 
  If that is not the case, then $M$ is not a model.
  
  \item $s = inprog(a,k)$: $M_{t} \models s$ implies that $M_{t} \models executed(a,t_{a})$. 
  Therefore, $M \models executed(a,t_{a})$. That means the rule in 
  Lines \ref{ag41}--\ref{ag42} is applicable and if 
  $\forall t_{a} \leq k < t_{a}+d(a), M \not\models inprog(a,k)$ then this 
  applicable rule is not applied. Therefore $M$ cannot be a model.
  
  \item $s = cmp(o(n,k_{1},a,dl'),k_{2})$: $M_{t} \models s$ implies that
  $M_{t} \models holdsat(o(n,k_{1},$ $a,dl'),k_{2})$ and also $M_{t} \models executed(a,k_{2})$.
  Consequently, $M \models holdsat$ $(o(n,k_{1},a,dl'),k_{2})$ and  $M \models executed(a,k_{2})$. 
  As a result, rule in Lines \ref{norm41}--\ref{norm42} is applicable and if 
  $M \not\models s$, the rule is not applied and $M$ is not a model. 
  
  \item $s = vol(o(n,k_{1},a,dl'),k_{2})$: $M_{t} \models s$ implies that
  $M_{t} \models holdsat(o(n,k_{1},$ $a,dl'),k_{2})$ and also $k_{2}=dl'$.
  Because $M$ is a model we know that $M \models holdsat(o(n,k_{1},a,dl'),k_{2})$.
 As a result, rule in Line \ref{norm62} is applicable and if 
  $M \not\models  vol(o(n,k_{1},a,dl'),k_{2})$, the rule is not applied and $M$ is not a model.
  
    \item $s = cmp(f(n,k_{1},a,$ $dl'),k_{2})$: $M_{t} \models s$ implies that
  $M_{t} \models holdsat(f(n,k_{1},$ $a,dl'),k_{2})$ and also $k_{2}=dl'$.
  Because $M$ is a model we know that $M \models holdsat(f(n,k_{1},a,dl'),k_{2})$ too. 
  As a result rule in 
  Lines \ref{norm101}--\ref{norm102} is applicable and if 
  $M \not\models  cmp(f(n,k_{1},a,dl'),k_{2})$, the rule is not applied and $M$ is not a model. 
    \item $s = vol(f(n,k_{1},a,dl'),k_{2})$: $M_{t} \models s$ implies that
  $M_{t} \models holdsat(f(n,k_{1},$ $a,dl'),k_{2})$ and also $M_{t} \models executed(a,k_{2})$.
  Consequently, we know that $M \models holdsat(f(n,k_{1},a,dl'),k_{2})$
  and  $M \models executed(a,k_{2})$. As a result rule in 
  Lines \ref{norm121}--\ref{norm122} is applicable and if 
  $M \not\models  vol(f(n,k_{1},a,dl'),k_{2})$, the rule is not applied and $M$ is not a model. 
  
  \item $s = terminated(x,k)$: $M_{t} \models s$ implies that 
  one of the following five situations hold:
  \begin{compactitem}
    \item \ref{pr82}: if this is the case then $M_{t} \models 
    executed(a,k-d(a)+1)$. If $M$ is a 
    model $M \models executed(a,k-d(a)+1)$. Thus, the rule in 
    Lines \ref{ag91}--\ref{ag92} is applicable and if $M \not\models terminated(s,k)$
    then the rule is applicable but not applied, which means that $M$ is not a  
    model. 
    
    \item \ref{pr14} and \ref{pr16}: if this is the case $x= o(n,k_{2},a,dl') \in s_{k}$, 
    and $M_{t} \models cmp(o(n,k_{1},a,dl'),k)$ or $M_{t} \models vol(o(n,k_{1},a,dl'),k)$. 
    Since $M$ is a model, then  if 
    $M_{t} \models cmp(o(n,k_{1},a,dl'),k)$ the same applies to $M$ and if 
    $M_{t} \models vol(o(n,k_{1},a,dl'),k)$ again the same applies to $M$.
    In either case, according to rule in Lines \ref{norm51}--\ref{norm52} and rule in Lines \ref{norm71}--\ref{norm72},
    $M \models terminated(o(n,k_{1},a,dl'),k)$ or $M$ is not a  model.  
    
    \item \ref{pr20} and \ref{pr22}: if this is the case $x= f(n,k_{1},a,dl')\in s_{k}$, 
    and  $M_{t} \models cmp(f(n,k_{1},a,dl'),k)$ or $M_{t} \models vol(f(n,k_{1},a,dl'),k)$. 
    Since $M$ is a model, if 
    $M_{t} \models cmp(f(n,k_{1},a,dl'),k)$ the same applies to $M$ and if 
    $M_{t} \models vol(f(n,k_{1},a,dl'),k)$ again the same applies to $M$.
    In either case, according to rule in Lines \ref{norm110}--\ref{norm111} and rule in Lines \ref{norm131}--\ref{norm132},
    $M \models terminated(f(n,k_{1},a,dl'),k)$ or $M$ is not a  model.  
  \end{compactitem} 
  
  \item $s=satisfied(g,k)$: $M_{t} \models s$ implies that $\forall x \in g^{+}, 
  x \in s_{k}$ and $\forall x \in g^{-}, x \not\in s_{k}$. If $M$ is a model 
  then according to the transformed version of rule \ref{goal1} 
  $M \models satisfied(g,k)$ and if that is not the case then $M$ is not a model.

\item $s=satisfied(g)$: $M_{t} \models s$ because $M_{t} \models 
satisfied(g,k)$. Since $M_{t} \models satisfied(g,k)$, according to previous 
item $M \models satisfied(g,k)$ too, therefore rule \ref{goal2} is applicable and 
$M \models satisfied(g)$ or it is not a model.
\end{compactitem}
The combination of these items demonstrates that $M$ cannot be a model for 
$\Pi_{PBase}^{M_{t}}$ if it differs from $M_{t}$. $M_{t}$ is therefore a minimal 
model for $\Pi_{PBase}^{M_{t}}$ and an answer set for $\Pi_{PBase}$.
\end{proof}

\bibliographystyle{elsarticle-harv} 


\end{document}